\documentclass{article}

\usepackage[preprint,nonatbib]{neurips_2024}

\usepackage{microtype}
\usepackage{graphicx}
\usepackage{placeins}
\usepackage{afterpage}
\usepackage{float}
\usepackage{booktabs} 
\usepackage{titlesec}

\usepackage{amsthm}
\usepackage{thmtools}
\usepackage{thm-restate}

\usepackage{url}

\PassOptionsToPackage{breaklinks}{hyperref}

\usepackage{fdk-typesetting}
\usepackage{booktabs}
\usepackage{relsize}

\usepackage{fdk-math}
\usepackage[nomarginkeys,biblatex-maxbibnames=10]{fdk-cite}
\usepackage{hyperref}

\usepackage{enumitem}

\usepackage{amsmath}
\usepackage{amssymb}
\usepackage{mathtools}

\theoremstyle{plain}
\newtheorem{theorem}{Theorem}

\Crefname{claim}{Claim}{Claims}
\Crefname{observation}{Observation}{Observation}

\newtheorem{lemma}[theorem]{Lemma}

\theoremstyle{definition}

\newtheorem{assumption}[theorem]{Assumption}
\theoremstyle{remark}

\newlist{itemlisting}{itemize}{3}
\setlist[itemlisting]{label={},leftmargin=2em}

\definecolor{ptred}{RGB}{187, 85, 102}
\definecolor{emred}{RGB}{204, 21, 17}
\definecolor{ptblue}{RGB}{0, 68, 136}
\definecolor{ptdarkyellow}{RGB}{153,119,0}
\definecolor{ptlightblue}{RGB}{102,153,204}

\usepackage{outlines}
\newlist{tightitemize}{itemize}{5}
\setlist[tightitemize]{parsep=0pt,label={$\bullet$},leftmargin=2em}

\newcommand{\prefigspace}{}%
\newcommand{\postfigspace}{}%

\newlist{contribs}{description}{1}
\setlist[contribs]{labelwidth=1em,itemindent=0pt,leftmargin=1em+6pt,labelsep=6pt}

\AtEveryBibitem{%
	\clearlist{address}
	\clearfield{eprint}
	\clearfield{isbn}
	\clearfield{issn}
	\clearlist{location}
	\clearfield{month}
	\clearfield{series}
 \ifentrytype{book}{}{%
	\clearlist{publisher}
	\clearname{editor}
 }
}

\newcommand{\ind}[1]{\mathbf{1}[#1]}

\DeclareCiteCommand{\myciteyear}
    {\usebibmacro{prenote}}
    {\bibhyperref{\printfield{year}}\bibhyperref{\printfield{extrayear}}}
    {\multicitedelim}
    {}

\newcommand{\textfrac}[2]{#1/#2}

\newcommand{\Diag}{\mathrm{Diag}}

\newcommand{\ddt}{\frac{\dif}{\dif t}}
\newcommand{\mdash}{\textrm{---}}
\usepackage{siunitx}

\usepackage{xpatch}

\usepackage{amsthm}
\usepackage{thm-restate}
\usepackage{hyperref}
\usepackage{cleveref}

\usepackage{xfrac}
\usepackage{subcaption}

\newcommand{\changed}[1]{{#1}}
\newcommand{\figImbalance}{%
\begin{figure}[!t]
\centering
\vspace{-.0em}
\prefigspace{}
\includegraphics{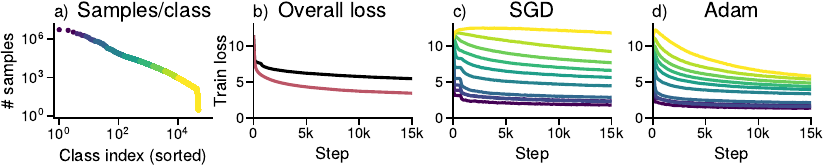}
\includegraphics{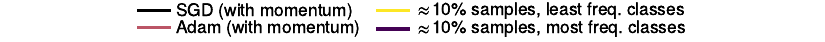}%
\postfigspace{}%
\caption{%
\textbf{Gradient descent does not make progress on low-frequency classes, while Adam does.}\\
Training GPT2-Small on WikiText-103. 
\textbf{(a)} Distribution of the classes sorted by class frequency, split into groups corresponding to ${\approx}10\%$ of the data.
\textbf{(b)} Overall training loss. 
\textbf{(c, d)} Training loss for each group using SGD and Adam.
SGD makes little to no progress on low-frequency classes
while Adam makes progress on all groups.
\textbf{(b)} is the average of \textbf{(c, d)}
for the respective~optimizer.
}
\label{fig:1}
\vspace{-.4em}
\end{figure}
}

\newcommand{\figMNIST}{%
\begin{figure}[!t]
\centering
\prefigspace{}
\includegraphics{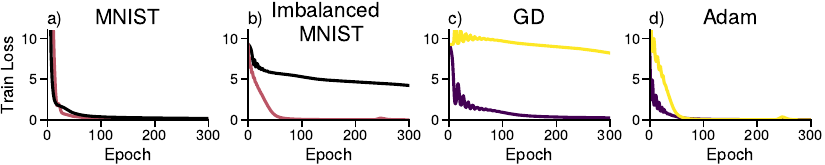}
\includegraphics{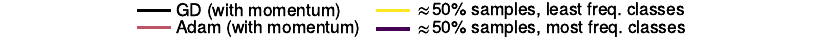}%
\postfigspace{}%
\caption{%
\textbf{Adam outperforms SGD for training a CNN under heavy-tailed class labels.}
\textbf{(a)}~Performance on the MNIST dataset.
\textbf{(b)} Performance on a modified MNIST with two groups of classes.
The first group consists of the 10 original classes with ${\approx5}$k samples each,
while the second consists of ${\approx}10$k added classes with $5$ examples each.
\textbf{(c, d)}~Performance of GD and Adam on the two groups.
}
\label{fig:image}
\end{figure}
}

\newcommand{\figResNet}{%
\begin{figure}[!t]
\centering
\prefigspace{}
\includegraphics{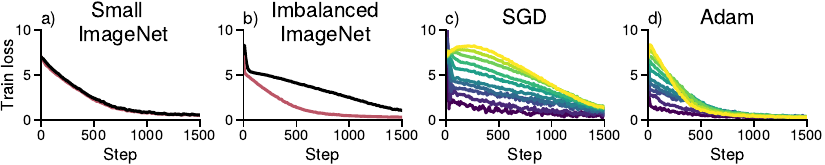}
\includegraphics{figs/plots/paper/legend/plot_legends_0_stoch.pdf}
\postfigspace{}%
\caption{%
\textbf{Adam outperforms SGD for training a ResNet under heavy-tailed class labels.}
\textbf{(a)}~Performance on a subset of ImageNet and
\textbf{(b)}~an imbalanced subset of ImageNet with class frequencies $\pi_k \propto 1/k$.
\textbf{(c, d)}~Performance of GD and Adam on groups corresponding to ${\approx}10\%$ of the data.
}
\label{fig:resnet}
\end{figure}
}

\newcommand{\figLinear}{%
\begin{figure}[!t]
\centering
\prefigspace{}
\includegraphics{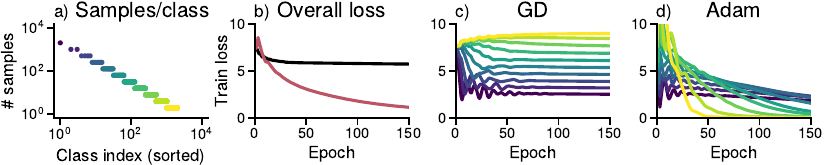}
\includegraphics{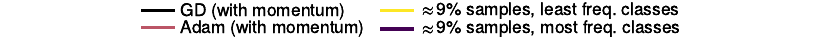}%
\postfigspace{}%
\caption{
\textbf{The impact of heavy-tailed class imbalance is reproducible with linear models.}
Softmax regression on synthetic data.
The inputs are drawn from a \changed{uniform distribution on $[0,1]^d$}.
The target classes are heavy-tailed \textbf{(a)} and independent of the inputs, 
but the model can still fit the data as it is overparameterized.
\textbf{(b, c, d)} Overall training loss and performance of GD and Adam on each subset.
}
\label{fig:linear}
\end{figure}
}

\newcommand{\figOpts}{%
\begin{figure}[!t]
\centering
\prefigspace{}
\includegraphics{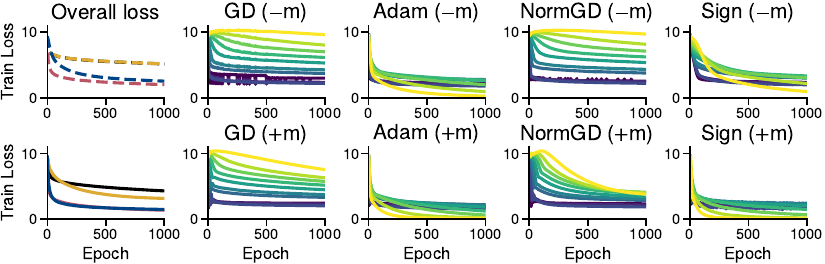}\hfill%
\includegraphics{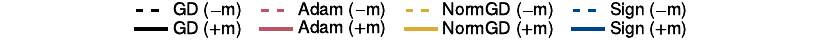}%
\postfigspace{}%
\caption{%
\textbf{Sign descent, as a simplified form of Adam, performs well on low-frequency classes.}
Training the last layer of a simplified one-layer transformer
with GD, Adam, normalized GD, and sign descent,
with and without momentum ($\pm$m).
Momentum and normalizing the magnitude help but have smaller effects
\smash{than using sign descent, which recovers similar dynamics to Adam.}
}%
\label{fig:opts}
\end{figure}
}

\newcommand{\figQuadratic}{%
\begin{figure}[!t]
\centering
\includegraphics{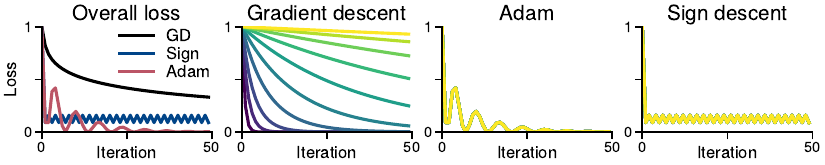}\\
\includegraphics{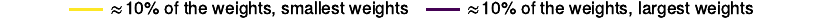}%
\postfigspace{}%
\caption{\textbf{Class-separation on the quadratic problem of \cref{sec:quadratic} with weights $\pi_k \propto 1/k$.}
GD fits functions with low weights more slowly, 
while Adam and sign descent have the same dynamics across all functions
and all the lines overlap
as every parameter $w_i$ is initialized at $w_i = 1$.
}
\label{fig:quadratic}
\end{figure}
}

\newcommand{\figGDAdam}{%
\begin{figure}[!t]
\centering
\includegraphics{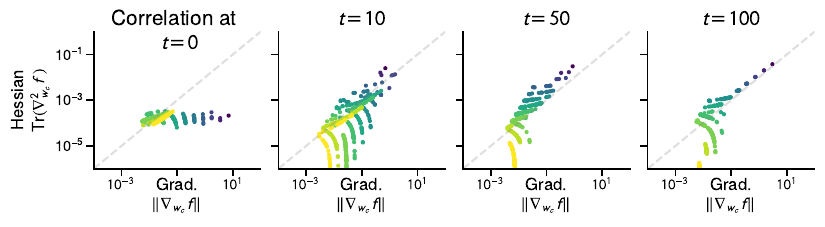}\\%
\includegraphics{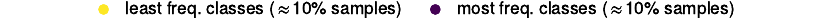}%
\postfigspace{}%
\caption{\textbf{The gradient norm and Hessian trace across blocks become correlated during training,}
over the path taken by Adam
in training the linear model of \cref{fig:linear}.
The blocks correspond to the rows $\vw_1, ..., \vw_c$ of the parameter matrix $\mW$.
The color indicates the class frequency, showing that  
lower (higher) frequency classes have smaller (larger) gradient norm and Hessian trace.%
}
\label{fig:grad-hess}
\end{figure}
}

\newcommand{\figHessian}{%
\begin{figure}[t]
\includegraphics{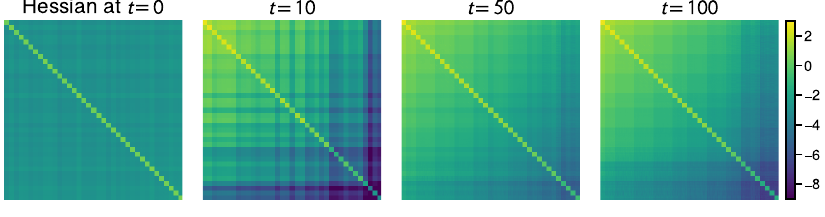}%
\caption{\textbf{The diagonal Hessian blocks are orders of magnitude larger than off-diagonal blocks.}
Showing the magnitude of a subset of the Hessian blocks ($\log_{10}(\abs{\Tr(\nabla^2_{ij} \Loss)})$)
for a $[160 \times 160]$ subset of the Hessian, sampling 40 classes log-uniformly and 40 input dimensions uniformly.
}
\label{fig:off-diag-blocks}
\end{figure}
}

\title{%
Heavy-Tailed Class Imbalance %
and Why Adam Outperforms Gradient Descent on Language Models   
}

\author{%
    Frederik Kunstner${}^1$\\
    \texttt{kunstner@cs.ubc.ca}
    \And
    Robin Yadav${}^1$\\
    \texttt{robiny12@student.ubc.ca}
    \And 
    Alan Milligan${}^1$\\
    \texttt{alanmil@cs.ubc.ca}
    \AND 
    Mark Schmidt${}^{1,2}$\\
    \texttt{schmidtm@cs.ubc.ca}
    \And
    Alberto Bietti${}^{3}$\\
    \texttt{abietti@flatironinstitute.org}
    \AND
    \vspace{-1em}
    \normalfont{${}^1$ University of British Columbia $\quad$ ${}^2$ Canada CIFAR AI Chair $\quad$ ${}^3$ Flatiron Institute}
}

\newcommand{\simplestImbalancedSetting}{%
\textbf{Simple imbalanced setting.}
\textit{%
Consider $c$ classes with frequencies $\pi_1, ..., \pi_c$
where all samples from a class are the same, $\vx_i = \ve_k$ if $y_i = k$, 
where $\ve_k$ is the $k$th standard basis vector in $\R^c\!\!.$
}}

\begin{document}

\maketitle

\begin{abstract}
    Adam has been shown to outperform gradient descent on large language models
    by a larger margin than on other tasks, but it is unclear why.
    We show that a key factor in this 
    performance gap is the heavy-tailed class imbalance 
    found in language tasks.
    When trained with gradient descent,
    the loss of infrequent words decreases more slowly 
    than the loss of frequent ones.
    This leads to a slow decrease on the average loss
    as most samples come from infrequent words.
    On the other hand, Adam and sign-based methods are less sensitive to this problem.
    To establish that this behavior is caused by class imbalance, 
    we show empirically that it can be reproduced across architectures and data types,
    on language transformers, vision CNNs, and linear models.
    On a linear model with cross-entropy loss,
    we show that class imbalance leads to 
    imbalanced, correlated gradients and Hessians
    that have been hypothesized to benefit Adam.
    We also prove that, in continuous time, 
    gradient descent converges slowly on low-frequency classes
    while sign descent does not.
\end{abstract}

\section{Introduction}

The recent success of large language models such as {GPT\nobreakdash-3}~\citep{brown2020gpt3} and its successors 
has relied on costly training procedures at unprecedented scale.
A key ingredient in their training is the Adam optimizer~\citep{kingma2015adam}, 
which outperforms stochastic gradient descent~(SGD) on language problems by a large margin.
Despite this large performance gap, we have a poor understanding of why Adam works better
and it has been difficult to find new optimizers that consistently improve over Adam \citep{schmidt2021descending}.
Not only is it computationally difficult to validate new optimizers on large models, 
but we also lack theoretical guidance;
we do not know what ``problem'' Adam solves to outperform SGD. 

The success of Adam on language transformers has been well documented.
Multiple works have found metrics or statistics that correlate with the improved performance of Adam,
showing that it
yields uniform updates across parameters despite imbalanced gradients \citep{liu2020transformers}, 
gives a better descent direction than the gradient \citep{pan2023toward}, 
and takes a path over which a robust variant of the condition number is smaller \citep{jiang2022adaptive}.
But these observations do not provide a mechanism explaining what property of the problem 
leads to the improved performance of Adam.

Plausible mechanisms have been put forward, but they do not provide a complete explanation.
\citet{zhang2020heavytails} show that Adam-like methods are more resilient to heavy-tailed noise, 
which seems more prominent in language than in vision tasks.
But noise is not the primary cause of the gap, as it already appears in deterministic training \citep{kunstner2023noise}.
An alternative hypothesis is that the magnitude of the gradient and Hessian are correlated,
which justifies clipping \citep{zhang2020relaxedsmoothness}.
But to justify methods that normalize element-wise, like Adam and sign-like methods, 
we additionally need the gradient and Hessian to be correlated across parameters \citep{crawshaw2022robustness}.
While there is empirical evidence for this behavior in neural networks, 
we do not have a good understanding of why this occurs, 
nor why this would be more pronounced on language rather than vision tasks.

\figImbalance

\subsection{Contributions}

\label{sec:contributions}

Our goal is to answer the following question:
\emph{what is the ``problem'' that makes SGD slow on language tasks, that Adam ``fixes'' to perform better?}

\textbf{We argue the problem is what we call heavy-tailed class imbalance,} 
where rare classes account for a large fraction of the data.
Language data is imbalanced as some words are much more frequent than others, 
typically following a power-law. 
A common modeling assumption is Zipf's law, 
where the $k$th most frequent word has frequency $\propto \textfrac{1}{k}$ \citep{piantadosi2014zipf}. 
For language tasks framed as next-token prediction, 
this property is reflected in the tokens and leads to heavy-tailed class imbalance.
This contrasts with typical vision datasets such as MNIST, CIFAR, and ImageNet, which are curated to have uniform classes, 
but also with imbalanced problems with a small number of classes.
For example, in binary classification, extreme imbalance implies the minority class has a limited impact on the loss;
with an imbalance of $99{:}1$, only $1\%$ of the data comes from the minority class.

\textbf{The performance gap arises because SGD makes slow progress on rare classes, see \cref{fig:1}.}
On a binary problem, slow performance on $1\%$ of the data 
need not have a large impact on the average loss if we make fast progress on the remaining $99\%$ of the samples.
In contrast, the heavy-tailed class imbalance found in language tasks
makes it possible for low-frequency classes to account for most of the data 
and significantly contribute to the loss, leading to slow performance overall.

\textbf{We show that heavy-tailed class imbalance makes SGD slow across tasks in \cref{sec:experiments}.}
We show that modifying vision datasets to exhibit heavy-tailed imbalance
leads to slow progress with SGD on architectures where the performance gap with Adam is typically smaller.
The impact of heavy-tailed imbalance can even be seen on linear models.
Additionally, the performance of SGD improves with techniques that address imbalance
such as upweighting rare classes.

\textbf{Our findings provide a simple model where Adam outperforms SGD,}
a softmax linear model under heavy-tailed class imbalance, which we analyze in \cref{sec:insights}.
We show empirically that 
a correlation between the magnitude of the gradient and Hessian across coordinates, 
used to justify the benefits of Adam, 
appears naturally even on a linear model with class imbalance. 
We provide intuition as to how this pattern emerges through an assignment mechanism
that leads to a correlation between class frequencies and the magnitude of the gradient and Hessian across parameters.
We additionally prove that,
on a simple dataset and in continuous time,
GD is slow on low-frequency classes while sign descent is insensitive to the class frequencies.

We do not claim that class imbalance is the only reason Adam outperforms SGD,
as other properties of the data or architectures likely also contribute to this gap.
Instead, we show that Adam consistently outperforms SGD under heavy-tailed class imbalance.
The difficulty of minimizing the loss of minority classes
has been explored for binary problems or problems few classes \citep{anand1993improved,francazi2023theoretical}, 
but the recent scaling of large language models to predictions over more than \num{100000} classes
puts the problem on a new scale.
Our findings indicate that heavy-tailed class imbalance has a significant impact on training performance
and should be a consideration for future optimizers to perform well on language 
and other tasks exhibiting heavy-tailed class imbalance.

\figMNIST

\begin{outline}

\section{Experimental results and ablation studies}
\label{sec:experiments}

\cref{fig:1} suggests a correlation between class frequencies and optimization performance 
that impacts SGD more than Adam.
The goal of this section is to verify that
(i)
class imbalance is a root cause for the performance gap between SGD and Adam, 
and (ii) whether this gap can be reproduced with simpler algorithms,
such as deterministic optimizers, or using sign descent as a proxy for Adam.

To test these hypotheses, we perform experiments focusing on the training loss 
as our objective is to understand what makes optimization difficult.
We use a simple training procedure, with a constant step-size tuned by grid search.
To show that stochasticity is not necessary to 
reproduce the impact of heavy-tailed class imbalance, 
we use deterministic updates (i.e., GD instead of SGD) on smaller models.
For visualization, we split the data into groups of classes with similar frequencies, as in \cref{fig:1}.
For instance, for 10 groups, the first group corresponds to ${\approx}10\%$ of the samples from the most frequent classes.
This grouping is only used for visualization and does not affect training.
The full details of the models, datasets and training procedures are described in \cref{apx:exp-details}.

\subsection{Reproducing the frequency gap with vision models}
\label{sec:mnist}

Language transformers are often contrasted with vision CNNs, 
where we do not see a large performance gap between SGD and Adam.
Our hypothesis is that a key differentiation between the two settings 
is the heavy-tailed class imbalance present in language data.
In this section, we show that making heavy-tailed vision datasets 
leads to slower performance with SGD and a larger performance gap with Adam.
These experiments show that heavy-tailed imbalance has a significant impact on performance 
and can make an otherwise ``easy'' problem into a ``hard'' one for SGD.

\textbf{CNN.}
We first use a CNN on a variant of MNIST with heavy-tailed class imbalance.
We augment the dataset to have two equally-sized groups of classes with a relative frequency difference of \num{1000}.
The first group consists of the original 10 classes with ${\approx}5$k samples/class.
For the second, we create ${\approx}10$k new classes with $5$ samples/class.
We create new classes by copying existing images and adding a ``barcode'' in a corner of the image, see \cref{apx:exp-details}.
The performance of GD and Adam is shown in \cref{fig:image}. 
On the original MNIST dataset, both optimizers drive the loss to 0. 
But on the imbalanced variant, GD makes almost no progress on half of the data 
corresponding to the low-frequency classes and progress stalls,
while Adam makes progress on both groups.

\textbf{ResNet.}
We replicate this effect with a ResNet18 on an imbalanced variant of ImageNet.
We subsample classes with frequencies $\pi_k \propto 1/k$ 
and compare against a uniform subset with a similar number of samples.
In \cref{fig:resnet}, we see that SGD and Adam perform similarly on uniform data 
but a performance gap appears across class frequencies on the heavy-tailed imbalanced dataset.
As in \cref{fig:1,fig:image}, SGD is slower on imbalanced data.

\textbf{Vision Transformers.}
This performance gap also appears with vision transformers (ViT).
In \cref{apx:vision}, we see that SGD and Adam perform similarly on ImageNet, 
but exhibit a similar performance gap on the imbalanced variant.
While ViTs may require more raw data , data augmentations, 
or regularization to generalize as well as ResNets \citep{steiner2022howtotrain},
there does not seem to be a large gap between SGD and Adam on optimization performance
without class imbalance.

\figResNet

\subsection{Reproducing the frequency gap with a linear model on uniform data}
\label{sec:linear}

To highlight that heavy-tailed imbalance alone 
can lead to the observed difficulties,
we reproduce this behavior on a simple setting: a softmax linear model with cross-entropy loss. 
We create a dataset where the class frequencies approximate $\pi_k \propto 1/k$
and draw $n$ samples uniformly from $[0,1]$ in $d$ dimensions, independently of the label.
While there is no relationship to learn, 
the optimization problem is still well posed and a linear model can separate the data
if $n \ll d$. 
As on the transformer of \cref{fig:1}, 
GD makes less progress on low-frequency classes than Adam, %
as shown in \cref{fig:linear}.

This example illustrates that a problem that might look innocuous at first
is hard to optimize with GD due to heavy-tailed imbalance, 
while the performance of Adam is less negatively impacted.
Nonetheless, imbalance alone is not sufficient to make GD slow. 
It is possible to generate pathological datasets 
with heavy-tailed imbalance where GD fits all classes fast, by making all the samples  (close to) orthogonal.
In this case, each sample is learned independently of the others, 
and there is no difference across classes. 
We give additional examples on the linear model in \cref{apx:more-info}.

\subsection{Interactions between optimizer and imbalance}
\label{sec:optims}

We have shown that heavy-tailed class imbalance 
can lead to different performance across class frequencies,
but it is not clear which component of the training process 
has the highest impact on this behavior.
We next experiment with simple algorithms to answer the following questions.
(i) Is the impact of class imbalance due to stochasticity, or does it happen with deterministic training?
(ii) Which component of Adam leads to an improved performance?
and (iii) If imbalance is the problem, can we improve the performance of SGD by reweighting the losses?

\textbf{Class imbalance already impacts deterministic optimization.}
A natural hypothesis to explain the impact of class imbalance 
is that it may be due to small batch sizes in SGD;
rare classes could be sampled less often, and thus learned more slowly.
On the other hand, stochasticity has been found to have little impact on the gap between SGD and Adam \citep{kunstner2023noise}.
Our experiments in \cref{fig:image,fig:linear,fig:opts} and further examples in \cref{apx:deterministic-transformer}
reproduce the dynamics of \cref{fig:1} with full batch GD and Adam, 
indicating the problem already arises in the deterministic setting.

\textbf{Adam and sign descent both perform well under imbalance.}
Following \citet{kunstner2023noise}, 
we check whether the benefit of Adam is due to a change in the magnitude of the update or its direction.
Changing the magnitude as in normalized GD 
is known to perform better on separable problems \citep{shpigel2019convergence},
while the benefits of Adam have been attributed 
to the change of direction close to sign descent \citep{tieleman2012rmsprop,balles2019dissecting}.
We compare the performance of GD, Adam, normalized GD and sign descent, with and without momentum, 
for training the last layer of a small transformer in \cref{fig:opts}
and on additional problems in \cref{apx:additional-opts}.
Normalization and momentum helps across problems,
but they have less impact on the performance gap across class frequencies than changing the update direction.
Sign descent and Adam have a similar performance.

\textbf{Upweighting low-frequency classes can help.}
Given our hypothesis that the performance gap between (S)GD and Adam is due to class imbalance, 
we expect interventions directly targeting imbalance to improve performance. 
In \cref{apx:upweight}, we show that 
upweighting the loss of low-frequency classes can improve the performance of SGD.
While reweighting is not complete solution as it changes the objective function,
this experiment supports the hypothesis that the optimization problem is due to heavy-tailed class imbalance.

\section{An investigation on linear models}
\label{sec:insights}

Heavy-tailed imbalance already leads to slow performance on the linear softmax model of \cref{fig:linear},
but we do not have a good understanding of why GD becomes slow while Adam is less affected.
In this section, we explore the effect of heavy-tailed class imbalance on the special case of softmax linear models, 
showing that it leads to correlated, imbalanced gradients and Hessians.
In \cref{sec:quadratic}, we give an example on a quadratic where imbalanced Hessians lead to a performance gap between GD and Adam.
In \cref{sub:toy_grad_hess}, we show that class imbalance leads to imbalanced gradients and Hessians 
that are correlated with class frequencies 
through an \emph{assignment mechanism}, showing that this pattern emerges naturally.
Finally, we prove that on a simple imbalanced problem and in continuous time,
GD is slow on low-frequency classes while sign descent is fast on all classes in \cref{subsec:sign}.

\subsection{Intuition on a weighted quadratic problem}
\label{sec:quadratic}

Consider the following toy problem
which is purposefully oversimplified to provide 
a high-level intuition about the optimization dynamics.
Suppose we have $c$ functions $f_1, ..., f_c$, 
corresponding to the losses for each class,
that are on the same scale
in the sense that gradient descent with step-size $\alpha$ makes fast progress on any $f_i$.
For concreteness, take $f_i(w) = \raisebox{.05em}{\scalebox{.8}{$\frac{1}{2}$}}\norm{w}{}^2$, 
where GD with a step-size of $1$ converges in one step.
Instead of running GD on each function independently,
suppose we run GD on the weighted average 
$f(w_1, ..., w_c) = \sum_{i=1}^c \pi_i f_i(w_i)$ 
with positive weights $\pi_1 \geq ... \geq \pi_c$, $\sum_i \pi_i = 1$,
corresponding to the class frequencies.
If these weights span multiple orders of magnitude,
we expect a similar behavior as in \cref{fig:1,fig:image,fig:resnet,fig:linear,fig:opts}, 
as illustrated in \cref{fig:quadratic}.
GD makes slow progress on functions with low weights
as the gradient w.r.t. $w_k$ is scaled by $\pi_k$,
\aligns{
    w_k^{(t)} 
    = w_k^{(t-1)} - \alpha \pi_k f_k'(w_k^{(t-1)}) 
    = (1-\alpha\pi_k)^t w_k^{(0)}.
}
This slow convergence on functions with low weights
cannot be fixed by increasing the step-size, 
as increasing it beyond $1 / \pi_1$ would cause instabilities on the highest-frequency ``class'' $f_1$.
The problem is that we use the same step size for all functions, which have different scales.
Adam and sign descent are less sensitive to this problem as their updates are independent of $\pi_k$, 
\aligns{
    w_k^{(t)} = 
    w_k^{(t-1)} - \alpha \frac{\pi_k f_k'(w_k^{(t-1)})}{\big|{\pi_k f_k'(w_k^{(t-1)})}\big|}
    =
    w_k^{(t-1)} - \alpha \sign(f_k'(w_k^{(t-1)})).
}
While sign descent or Adam with a fixed step-size need not converge 
and can oscillate around the minimum, 
they perform much better in early iterations, independently of $\pi_k$.

\figLinear

Another perspective is that the imbalance in the weights $\pi_1, ..., \pi_c$ makes the problem ill-conditioned.
The weights not only affect the gradient of $f$ but also its Hessian, which is $\Diag([\pi_1, ..., \pi_c])$. %
A common intuition for Adam is that using the magnitude of the coordinates of the gradient as a preconditioner 
is a good proxy for the Hessian diagonal \citep{duchi11adaptive,kingma2015adam},
which would also lead to larger step-sizes for coordinates with small $\pi_k$. 
While this does not hold in general \citep{kunstner2019limitations}, 
the gradient can be a reasonable approximation to the Hessian on this problem.
The gradient is $[\pi_1 w_1, ..., \pi_c w_c]$.
If the weights $\pi_1, ..., \pi_c$ vary by orders of magnitude more than the parameters $\abs*{w_1}, ..., \abs*{w_c}$, 
the gradient and Hessian will be correlated,
and preconditioning by the gradient magnitude or Hessian diagonal will yield similar directions.

\figOpts

\subsection{Correlations between the magnitude of the gradient and Hessian across coordinates}
\label{sub:toy_grad_hess}

What is lacking to explain Adam's improved performance
is an understanding of how a correlation between the gradient and Hessian arises in realistic problems.
This feature has been observed on neural networks, 
but we do not yet know why it appears, even on the softmax linear problem.
The caricature of the diagonal quadratic problem of the previous section provides some intuition, 
but does not directly apply to the softmax linear model of \cref{fig:linear} 
as that problem is neither quadratic nor separable.
Nonetheless, a similar pattern emerges in the rows $\vw_1, ..., \vw_c$ of its parameter matrix~$\mW$;
the magnitude of the gradient and Hessian across rows and the class frequencies can become correlated during training due to class imbalance.
In this section, we establish this observation empirically and provide a mechanism for how it emerges.

In \cref{fig:grad-hess}, 
we show the gradient norm against the Hessian trace for each block $\vw_k$
throughout the trajectory of Adam on the softmax linear model of \cref{fig:linear}.
While there is no correlation at initialization, 
the gradient and Hessian blocks become correlated with class frequencies during training
and become imbalanced.
The diagonal blocks are also orders of magnitude larger than off-diagonal blocks, 
as shown in \cref{fig:off-diag-blocks}, indicating a weak dependence across blocks.
Similar dynamics occur with GD, although only on high-frequency classes
as GD makes little progress on low-frequency classes, 
and in the last layer of deep networks, 
see \cref{apx:detailled-dynamics}.

To explain this behavior, we show that the impact of samples on the Hessian follows an \emph{assignment mechanism}:
if the model assigns samples to their correct class, 
the Hessian with respect to $\vw_k$ is primarily influenced by samples from class $k$, 
leading to a correlation between the magnitude of the gradient, Hessian, and class frequencies.
To capture this effect, 
we introduce some notation and a simplifying assumption. 
Suppose we have $n$ samples with inputs $\vx_i \in \R^d$ and labels $\smash{y_i} \!\in\! [c]$,
where class $k$ has frequency $\pi_k = \nicefrac{n_k}{n}$.
The parameters of the linear model are $\mW \!\in \smash{\R{}^{c \times d}}$.
We write $\vp(\vx) = \sigma(\mW \vx)$ for the predicted probabilities where $\sigma$ is the softmax,
and summarize the data as
\aligns{
    \textstyle \bar \vx = \frac{1}{n}\sum_{i=1}^n \vx_i, &&
    \textstyle \bar \vx^k = \frac{1}{n_k}\sum_{i : y_i = k} \vx_i, &&
    \textstyle \bar \mH = \frac{1}{n}\sum_{i=1}^n \vx_i\vx_i^\top\!, &&
    \textstyle \bar \mH^k  = \frac{1}{n_k}\sum_{i : y_i = k} \vx_i\vx_i^\top\!.
}
\vspace{-1.6em}
\begin{assumption}[correct assignment]\label{ass:correct}
    The model correctly assigns samples to class~$k$ 
    if it predicts~$k$ with non-negligible probability~$p$ on samples from that class
    ($\vp(\vx_i)_k = p = \omega(\nicefrac{1}{c})$ for $\vx_i$ from class $y_i = k$), 
    and predicts $k$ with near-random chance otherwise
    ($\vp(\vx_i)_k = O(\nicefrac{1}{c})$ for $\vx_i$ where $y_i \neq k$).
\end{assumption}
\vspace{-.3em}
\begin{restatable}{proposition}{propassignment}\label{prop:assignment}
If initialized at $\mW_0 = 0$, 
the gradient and Hessian of the loss $\Loss$ w.r.t. $\vw_k$ are
\alignn{
    \label{eq:grad-hess-init}
    \nabla_{\vw_k} \Loss(\mW_{\!0})
    &= \textstyle
    \pi_k \bar\vx^k - \frac{1}{c}\bar\vx,
    &
    \nabla_{\vw_k}^2 \Loss(\mW_{\!0})
    &= \textstyle
    \frac{1}{c}\paren{1-\frac{1}{c}} \bar \mH,
}
During training, if the model correctly assigns samples to class $k$ with probability $p$ (\cref{ass:correct}),
\alignn{\label{eq:assignment}
    \!\!\!\!
    \begin{aligned}
        \nabla_{\vw_k} \Loss   &= \textstyle \phantom{p}(1-p) \pi_k \, \mathrlap{\bar\vx^k}\phantom{\bar\mH^k} + O\paren{\frac{1}{c}}, \\
        \nabla_{\vw_k}^2 \Loss &= \textstyle p (1-p) \pi_k \, \bar\mH^k + O\paren{\frac{1}{c}},
    \end{aligned}
    &&
    \text{ and } \quad
    \norm{\nabla_{\vw_k}\Loss}
    \sim
    \paren{
        \frac{1}{p} \frac{\norm{\bar\vx^k}}{\Tr(\bar\mH^k)}
    }
    \Tr(\nabla^2_{\vw_k}\Loss)
    \text{ as } c \to \infty,
}
for classes where the frequency does not vanish too quickly, $\pi_k = \omega(\nicefrac{1}{c})$.
\end{restatable}

The assumption that $c \to \infty$ is used to obtain a simple and interpretable equation in the correlation.
In practice, $c > 10^3$ appears sufficient to make the dependence on $\pi_k$ 
appear, as in \cref{fig:grad-hess}.

At initialization, \cref{eq:grad-hess-init} shows that
the Hessian blocks are uniform across classes while the gradients depend on $\pi_k$.
If the data is uniform across classes ($\smash{\norm*{\bar\vx{}^k} \,{\approx}\, \norm*{\bar\vx{}^{k'}\!}}$)
while the frequencies differ by orders of magnitude,
the the gradient blocks will mirror the class frequencies
for high-frequency classes where $\pi_k \gg \nicefrac{1}{c}$.
This confirms the pattern observed at initialization in \cref{fig:grad-hess}.
During training, 
\cref{eq:assignment} indicates a correlation between gradient norm and Hessian trace 
if classes have similar values of ${\norm*{\bar\vx^k}}$, ${\Tr(\bar\mH^k)}$ and predicted probabilities $p$, 
confirming the behavior observed during training in \cref{fig:grad-hess} for the high frequency classes.
As Adam fits low-frequency classes faster in \cref{fig:linear}, 
they have a value of $p$ closer to $1$ (shown in \cref{apx:detailled-dynamics})
and deviate slightly from the trend in \cref{fig:grad-hess}, as expected from \cref{eq:assignment}.

We now give the main intuition and defer the derivation of the asymptotics to \cref{apx:theory}.
We ignore off-diagonal blocks here,
as they are orders of magnitude smaller than diagonal blocks (\cref{fig:off-diag-blocks}), 
and show in \cref{apx:off-diag} that they are expected to be small.
\vspace{-1em}
\begin{proof}[Proof idea]
Our loss is $\Loss(\mW) = \frac{1}{n}\sum_{i=1}^n \ell(\mW\!, \vx_i, \vy_i)$, where $\ell$ is a softmax linear model,
\alignn{
    \ell(\mW\!, \vx, y) =
    -\log(\sigma(\mW\vx)_{y}), 
    \quad \text{ with } \quad
    \sigma(\vz)_k =
    \tfrac{\exp(\vz_k)}{\sum_{j} \exp(\vz_{j})}.
}
Writing $\vp(\vx) = \sigma(\mW\vx)$ for the vector predicted probabilities, the gradient and Hessian blocks are 
\alignn{
    \label{eq:grad-and-hess}
    \nabla_{\vw_k} \ell(\mW\!, \vx, y)
    =
    (\ind{y=k} - \vp(\vx)_k)\vx,
    &&
    \nabla_{\vw_k}^2 \ell(\mW\!, \vx, y)
    = 
    \vp(\vx)_k (1 - \vp(\vx)_k)\vx \vx^\top\!.
}
The contribution of a sample $(\vx, y)$ to the gradient w.r.t. $\vw_k$
primarily depends on whether the sample belongs to class $k$ through the $\ind{y=k}$ term, 
while the contribution to the Hessian block depends 
on whether the model assigns that sample to class $k$ through $\vp(\vx)_k$.
At initialization, $\vp(\vx)_k = 1/c$ for all samples, 
and averaging the terms in \cref{eq:grad-and-hess} yields \cref{eq:grad-hess-init}.
Highlighting this effect during training is more challenging due to the dependency on the predictions.
However, if $\mW$ start to assign samples to their correct classes (\cref{ass:correct}), 
we can obtain a similar decomposition as \cref{eq:grad-hess-init}.
For a given class $k$, the probabilities for correct labels are all $p$
while the probabilities for incorrect ones are bounded by $O(1/c)$,
which vanishes in the limit of $c \to \infty$.
\end{proof}
\vspace{-1em}
This assignment mechanism explains why the gradient, Hessian and class probabilities can become correlated on the linear model.
While the gradient does not directly approximate the Hessian, 
the main feature of the imbalance in the Hessian comes from the weighting by the class frequencies $\pi_1, ..., \pi_c$,
which is present in both the gradient and the Hessian,
as shown in \cref{fig:grad-hess,fig:off-diag-blocks}.
This correlation is not a global property of the problem, 
as there are parameters for which the opposite pattern holds, see \cref{apx:detailled-dynamics}, 
but it appears during training if the optimization algorithm makes progress.
While the per-coordinate normalization 
of Adam or sign descent was not designed to specifically address class imbalance, 
they appear to benefit from this property to make faster progress.

\figQuadratic 

Our results complement prior work on optimization with class imbalance on problems with two or few classes,
which argued that the gradient is dominated by the majority class, 
and as a result is biased towards making progress on the majority class at the expense of the minority class~\citep{anand1993improved,ye20221procrustean,francazi2023theoretical}.
While this explains why GD might not make fast progress on rare classes, 
it is not clear why this would lead to slow performance on average,
especially under heavy-tailed imbalance where there is no ``majority''.
Our results show that, in addition to imbalance in the gradients,
class imbalance leads to optimization difficulties through imbalanced Hessians.

\subsection{Improvement of sign-based approaches over gradient descent}
\label{subsec:sign}

While the above arguments provide a high-level intuition
as to why the gradient might be a reasonable proxy for the Hessian, 
it remains difficult to formally describe this effect 
and prove the benefits of Adam over GD without strong assumptions.
Doing so would require a fine-grained analysis of the dynamics,
as the correlation only appears during training.
To obtain a provable a guarantee highlighting the benefit of sign-based methods, 
we consider a stripped-down problem where the only difficulty lies in the class imbalance:

\simplestImbalancedSetting{}

This setting is trivial as a possible solution is 
$\mW = \alpha \mI$ with $\alpha \to \infty$,
or taking one step of gradient descent with an arbitrarily large step-size.
However, we will see that the dynamics with small step-sizes
already exhibit the separation by class frequencies observed experimentally. 
In this simplified setting,
we show that the continuous time variant of gradient descent, gradient flow, 
and sign descent as a proxy for Adam,
obtain qualitatively different convergence rates (proof in \cref{apx:continuous-time}).
\begin{restatable}{theorem}{thmgflow}\label{thm:gflow}
On the \emph{simple imbalanced setting},
gradient flow and continuous time sign descent
initialized at $\mW = 0$ minimize the loss of class $k$,
$\ell_k(t) = -\log(\sigma(\mW(t)\ve_k)_k)$, at the rate
\aligns{
    \text{Gradient flow: } \quad \ell_k(t) = \Theta\paren{{1}/{\pi_k t}},
    &&
    \text{Continuous time sign descent: } \quad \ell_k(t) = \Theta\paren{e^{-ct}}.
}
\end{restatable}
\vspace{-.5em}
The difference between the sublinear rate of gradient flow ($1/t$) 
and linear rate of sign descent ($e^{-t}$)
is similar to existing results for overparameterized logistic regression, 
where normalized updates converge faster 
as they keep increasing the margin despite small gradients \citep{shpigel2019convergence}.
The novel element is that the convergence of gradient flow strongly depends on the class frequencies $\pi$,
while the convergence of sign descent is independent of the class frequencies.

\figGDAdam

This setting is admittedly oversimplified
and does not capture some of the features observed in our experiments.
For example, in \cref{thm:gflow}, the loss is monotonically decreasing for all classes. 
This no longer holds once we introduce a bias term
and the loss from low-frequency classes will instead first increase,
as can be seen for example in \cref{fig:linear}.
This setting is also biased towards sign descent, as the inputs are aligned with the basis vectors.
Finally, the problem is inadequate to study large step-sizes, 
as it can be solved in one large step.
On data with non-orthogonal classes, 
large step-sizes would lead to training instabilities
and oscillations in the loss of frequent classes, 
as can be seen in \cref{fig:image,fig:resnet,fig:linear,fig:opts}.
Nevertheless, this result formally establishes the benefit of sign-based updates
and we believe it captures the key difficulty encountered by GD under heavy-tailed class imbalance.

\figHessian

\section{Discussion and limitations}
\label{sec:limitations}

\textbf{Interaction with stochasticity.}
Our experiments include both stochastic and deterministic training regimes 
and show that stochasticity is not the cause of the slow performance of SGD on low-frequency classes, 
as it already appears between full batch GD and Adam. 
This observation is consistent with prior work 
showing that the performance gap between SGD and Adam on language transformers 
already appears with deterministic training \citep{kunstner2023noise}.
However, we do not attempt to quantify the interaction between stochasticity and class imbalance
and leave it for future work.

\textbf{Training performance vs. generalization.} 
Our main focus is on optimization performance.
Our observations need not generalize to the validation loss, 
especially in settings prone to overfitting,
as good training performance may lead to overfitting on classes with few samples \citep{sagawa2020investigation}
However, some form of memorization might be needed in long-tailed settings \citep{feldman2020does},
and if SGD cannot even fit the training data, generalization cannot be good.
On the transformer of \cref{fig:1}, we observe similar dynamics across frequencies on the validation loss, shown \cref{apx:fig1-validation}.
Training dynamics on the empirical and population loss are also often similar, 
particularly early in training~\citep[see, e.g.,][]{nakkiran2021deep,ghosh2022three},
and the one-pass training regime commonly used in large language models 
might mitigate those issues by blurring the line between train and test loss.

\textbf{Additional difficulties due to text data.}
We study the effect of the distribution of the classes, the \emph{next} token to be predicted, 
but other optimization difficulties might arise from the heavy-tailedness of text data.
For example, 
the sequence of tokens used as inputs to the embedding layer are also heavy-tailed.
This imbalance might lead to slow progress on embeddings for rare tokens with GD,
giving another potential cause for a performance gap.
Full sentences~\citep{ryland2015zipf} 
and latent rules or mechanisms required to understand a paragraph~\citep{michaud2023quantization} 
may also display heavy tails, and Adam could be beneficial 
if those are captured by intermediate layers~\citep[e.g.,][]{meng2022locating,wang2022interpretability,bietti2023birth}.
The choice of tokenization has also been shown to impact downstream performance, 
which has been attributed to the lack of samples on rare tokens~\citep{gowda2020optimalvocab}
and the improved efficiency of more uniform tokenizers \citep{zouhar2023tokenization}.
Our results indicate that tokenization also has a large impact on optimization performance.

\textbf{Difficulties due to architectures.}
Beyond the class distribution, additional optimization difficulties may arise from the architectures,
due to depth, signal propagation~\citep{noci2022signal,he2023deep},
vanishing gradients and higher order derivatives~\citep{liu2020transformers,orvieto2022vanishing}. 
The simplified transformer of \citet{ahn2023linearattention} 
also exhibits many of the difficulties observed in the literature on regression instead of a classification problem. 
However, a phenomenon similar to the assignment mechanism could still explain the benefit of Adam.
The oscillations in the loss observed at the feature level by \citet{rosenfeld2023outliers}
suggests a link between subsets of the samples and subsets of the parameters.
For example, if a convolution filter detects a specific background color 
and captures a specific feature of the data, 
the magnitude of the gradient and Hessian at intermediate layers could 
be influenced by the relative frequency of the feature in the data, 
leading to another form of imbalance.

\section{Conclusion}
\label{sec:discussion}

We have shown that heavy-tailed class imbalance 
leads to a performance gap between (S)GD and Adam.
This effect is reproducible across architectures and data types,
but is most salient on language tasks which naturally exhibit heavy-tailed imbalance. 
As vision tasks are typically more uniform,
imbalance is a key differentiating feature of the training difficulties in language tasks.
The correlation between gradient and Hessian blocks 
that occurs due to class imbalance
provides a simple setting that justifies the intuition 
that Adam-like algorithms can ``adapt to curvature''.
We provide an explanation for how this correlation arises during training
and prove on a simplified problem that gradient descent performs poorly on low-frequency classes
while sign descent is unaffected by class frequencies.

\end{outline}

%
%
%
%
%
%

\subsection*{References}
\AtNextBibliography{\normalsize}
\printbibliography[heading=none]

\newpage

\appendix

{\hrule height 4pt \vskip 0.25in \vskip -\parskip}
{\centering \LARGE\bf Supplementary Material \par}
{\vskip 0.29in \vskip -\parskip \hrule height 1pt \vskip 0.09in}

\vspace{2em}

\newcommand{\customtoc}[1]{
    \item[\ref{#1}] \nameref{#1} \dotfill \pageref{#1}
}
\begin{itemize}[leftmargin=2em]
    \customtoc{apx:exp-details}
    \customtoc{apx:vision}
    \customtoc{apx:more-info}
    \customtoc{apx:deterministic-transformer}
    \customtoc{apx:additional-opts}
    \customtoc{apx:upweight}
    \customtoc{apx:detailled-dynamics}
    \customtoc{apx:theory}
    \customtoc{apx:continuous-time}
\end{itemize}

\section{Experimental details}
\label{apx:exp-details}

This section documents the datasets, models, software and experimental setup.
The code will be made available at \texttt{(to be released)}.

\newcommand{\crunchspace}{}

\crunchspace

\subsection{Datasets}\label{apx:datasets}

\begin{itemize}[nosep,parsep=4pt,leftmargin=1em]
    \item 
    \textbf{WikiText-103} \citep{Merity2017Wikitext}, 
    using sequences of $1\,024$ tokens and the BPE tokenizer \citep{senrich2016bpe}, with a vocabulary of size $50\,608$.
    \item 
    \textbf{WikiText-2} \citep{Merity2017Wikitext} is used in \cref{apx:tokenizer}
    to illustrate that other combinations of datasets and tokenizers lead to heavy-tailed distributions.
    \item 
    \textbf{PTB} \citep{marcus1993ptb}, using sequences of $35$ tokens
    built from a word-based tokenizer (\texttt{basic\_english} provided by \texttt{torchtext}), 
    for a vocabulary of size $9\,920$.
    For deterministic runs, we use the validation set as a reduced training set, labeled \textbf{TinyPTB}.
    \item \textbf{MNIST} \citep{Lecun1998LeNet}.
    \item \textbf{ImageNet} \citep{deng2009imagenet}.
\end{itemize}

\crunchspace

\subsection{Custom datasets}

\newcommand{\RHTL}{Random Heavy-Tailed Labels}

\begin{itemize}[nosep,parsep=4pt,leftmargin=1em]
    \item
    \textbf{The \RHTL{} dataset}
    is a synthetic dataset exhibiting heavy-tailed class imbalance.
    The number of samples per class and the number of classes 
    are picked to approximate a power-law distribution.
    We create $m$ ``groups'' of classes,
    where each class within a group has the same relative frequency;
    \aligns{
        \underbrace{1 \text{ class with } 2^m \text{ samples},}_{\text{group 1}} \quad
        \underbrace{2 \text{ classes with } 2^{m-1} \text{ samples},}_{\text{group 2}} \quad
        \ldots, \quad
        \underbrace{2^{m-1} \text{ classes with } 2 \text{ samples.}}_{\text{group $m$}}
    }
    The inputs are drawn from a uniform distribution on $[0,1]$, 
    independently of the class label.
    The inputs are in $d = (m+1) \, 2^{m}$ dimensions,
    the number of samples is $n = m \, 2^{m}$
    and the number of classes is $c = 2^{m+1}-1$.
    We use two variants of the datasets;
    a large one in \cref{fig:linear}, \cref{apx:additional-opts}
    ($m=11, n = 22\,528, d = 24\,576, c = 4\,095$)
    and a small one in \cref{apx:more-info}
    ($m = 8, n = 2\,048, d = 2\,304, c = 511$).
    \item 
    \textbf{The Barcoded MNIST dataset} is a modified variant of MNIST.
    We start with 50k examples from the original MNIST dataset across 10 classes, 
    and create $51\,200$ ($5 \times 10 \times 2^{10}$) new images.
    The new examples are copies of existing image with an added ``barcode'',
    a 10-bit number encoded in a corner of the image, as in the examples below.
    The class label is a combination of the original class and this barcode.

    ~\null\hfill \includegraphics[width=\textwidth]{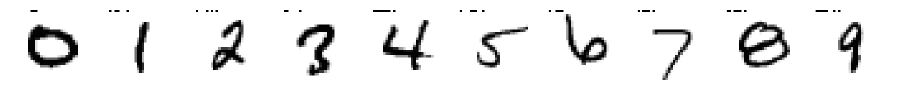}\hfill
    
    The \textbf{Barcoded-only} dataset 
    contains $10 \times 2^{10}$ classes with $5$ samples each.
    To obtain an imbalanced dataset, 
    we combine the barcoded images with the original samples from the MNIST dataset
    to get $101\,200$ examples spread across $10\,250$ ($10 \times 2^{10} + 10$) classes classes; $10\,240$ with 5 examples per class
    and $10$ classes with $\approx 5k$ examples per class, labeled \textbf{MNIST+Barcoded}
    \item
    \textbf{The Heavy Tailed ImageNet} 
    dataset is a subset of ImageNet \citep{deng2009imagenet},
    subsampled to exhibit heavy-tailed class imbalance. 
    We sort the original $1000$ classes by frequency
    and sample $\ceil{\frac{1300}{k}}$ images from the $k$th class,
    leading to $n = 10\,217$ samples.
    \item
    \textbf{The small ImageNet} dataset 
    is a uniform subset of ImageNet 
    to contrast the with the heavy tailed variant.
    We sample 10 images per class to get $n = 10\,000$ samples.
\end{itemize}

\crunchspace

\subsection{Models}
\label{apx:models}

\begin{itemize}[nosep,parsep=4pt,leftmargin=1em]
    \item
    \textbf{The 2-layer transformer} used in  \cref{apx:deterministic-transformer}
    is a transformer \citet{vaswani2017attention},
    based on the PyTorch implementation of \texttt{TransformerEncoderLayer} \citep{paszke2019pytorch}.
    \aligns{
        \text{
        Embedding $\to$ $2{\times}$ [Attention $\to$ Linear $\to$ ReLU $\to$ Linear] $\to$ Classifier.
        }
    }
    The model includes LayerNorm, dropout, and skip connections 
    \citep{he2016deep,ba2016layernorm,srivastava2014dropout}.
    The embedding dimension and width of the linear layers is 1000
    and the attention modules use 4 heads.
    \item \textbf{The simplified transformer} used in \cref{fig:opts} and \cref{apx:deterministic-transformer}
    does not use encoder blocks, and only uses attention:
    \aligns{
        \text{
        Embedding $\to$ Attention $\to$ Classifier.
        }
    }
    We remove LayerNorm, dropout, and the block 
    [Linear $\to$ ReLU $\to$ Linear] containing the non-linearity.
    In \cref{fig:opts}, we freeze the embedding and attention layers at initialization, 
    and only the last classification layer is trained.
    The model is then a linear model on a fixed feature transformation.
    \item \textbf{The GPT2-Small} model \citep{radford2019gpt2} is used in \cref{fig:1}.
    The blocks includes LayerNorm, residual connections, 
    and dropout on the embedding and dense layers. 
    We use sinusoidal encoding as in the transformer architecture \citep{vaswani2017attention}. 
    The embedding dimension is 768, the width of the intermediate layers is 3072, 
    and we use 12 encoder blocks with 12 self attention heads.
    \item 
    \textbf{The convolutional network} used in \cref{fig:image} and \cref{apx:vision}
    is a 2-layer convolution
    \aligns{
        \text{
            Conv $\to$ Relu $\to$ MaxPool $\to$
            Conv $\to$ Relu $\to$ MaxPool $\to$ Linear
        }
    }
    \item 
    \textbf{The linear model} 
    used in \cref{fig:linear,fig:grad-hess} and \cref{apx:more-opt} uses a bias vector.
    \item 
    \textbf{The ResNet18 }  model \citep{he2016deep} is used in \cref{fig:resnet}. Additionally, a  variant replacing the BatchNorm layers with LayerNorm is used in \cref{apx:vision}.
    \item
    \textbf{The SimpleViT} model \citep{beyer2022simplevit} used in \cref{apx:vision-resnet} 
    follows the architecture of a ViT-S/16 \citep{touvron2021efficient}, 
    based on the \texttt{vit-pytorch} implementation
    (\url{https://github.com/lucidrains/vit-pytorch} \texttt{v1.6.5}).
\end{itemize}

\subsection{Training procedures}
\label{apx:training}

Our primary focus is on the performance 
of the optimizers on the training error,
using the simplest training procedure possible. 
We use a constant step-size throughout training, set by grid search. 
We start with a sparse grid of powers of 10 $[10^{-6}, 10^{-2}, ..., 10^{1}]$
and increase the density to half-powers around the best step-size.
The step-size is selected to minimize the 
maximum over 3 seeds of the training loss at the end of training.
For some settings, this selection still produces runs that are unstable;
the training loss is the smallest at the end 
but oscillates a lot during training, reaching loss values that are worse than at initialization. 
For those runs, we use the next smaller step-size, 
which has similar performance but is more stable.

We use gradient accumulation (computing the gradient through multiple passes)
to achieve the following batch sizes;
\begin{itemize}[leftmargin=1.0em,label={-},parsep=2pt,itemsep=0pt,topsep=0pt]
    \item 
    The large transformer experiment in \cref{fig:1} uses mini-batches of 512 sequences of 1024 tokens.
    \item 
    The stochastic experiments with a smaller transformer in \cref{apx:deterministic-transformer}
    uses mini-batches of $512$ sequences of $35$ tokens. 
    \item  Both ResNet18 variants and the Simple Vision Transformer were trained using mini-batches of $1024$. The training images were normalized and randomly cropped to $224\times224$ pixels as is standard for ImageNet training.
    \item 
    Other experiments use the entire dataset to compute updates
\end{itemize}
Our experiments ran on a cluster using a mix of A100, P100, V100, and H100 GPUs.
The large scale experiment in \cref{fig:1} took $3$ days on a H100, 
while all other experiments ran in $2$--$8$ hours.
The total amount of compute used for this project is ${\approx}$3 GPU-years,
including preliminary experiments.

\subsection{Summary of settings used}
\vspace{-1em}
\begin{table}[!h]
\caption{Summary of models, datasets and batch-size used}
\begin{tabular}{@{}llll@{}}
\toprule
\textbf{Model}                  & \textbf{Dataset}              & \textbf{Batch size}               & \textbf{Used in}  \\ 
\midrule
GPT2-Small             & WT103                                    & 512                      & \cref{fig:1} and \cref{fig:1-validation}                                                                            \\
2-layer transformer    & PTB                                      & {512}  & \cref{fig:smaller-transformer-stochastic,fig:reweighting,fig:grad-hess-apx-ptb}                                                                         \\
1-layer transformer & TinyPTB                                  & Full                     & \cref{fig:smaller-transformer-deterministic,fig:tinyptb-more-opts}                                 \\
1-layer transformer & TinyPTB                                  & Full                     & \cref{fig:opts} (last layer only)                                  \\
CNN                    & Barcoded MNIST                           & Full                     & \cref{fig:barcoded-mnist-only}                                                                                                                         \\
CNN                    & MNIST                    & Full                     & \cref{fig:image,fig:barcoded-mnist-only}   \\
CNN                    & MNIST+Barcoded                    & Full                     & \cref{fig:image,fig:barcoded-mnist-only,fig:cnn-barcode-more-opts,fig:reweighting,fig:grad-hess-apx-mnist}  \\
Linear       & Random HT labels, m=11                                     & Full                     & \cref{fig:linear,fig:grad-hess,fig:linear-more-opts,fig:reweighting,fig:grad-hess-apx,fig:apx-inverse-correlation,fig:off-diag-lin}                                  \\
Linear       & Random HT labels, m=7                                      & Full                     & \cref{fig:input-data,fig:training-longer}                                                                                                             \\
Simple ViT             & ImageNet                                 & {1024} & \cref{fig:balancedvit}                                                                                                                                  \\
ResNet18               & Small and HT ImageNet & {1024} & \cref{fig:resnet,fig:reweighting,fig:grad-hess-apx-resnet}                                                                                             \\
ResNet18+LN     & Small and HT ImageNet & {1024} & \cref{fig:lnresnet}                                                                                                                                    \\
Simple ViT             & Small and HT ImageNet & {1024} & \cref{fig:imbalancedvit}                                                                                                                                \\ 
\bottomrule
\end{tabular}
\end{table}

\subsection{Optimization algorithms}
Given momentum buffers $m_t$ initialized at $m_0 = 0$ and a (possibly) stochastic gradient $\tilde g_t$, 
we implement the update of GD, normalized GD and sign descent 
with heavy-ball momentum as 
\aligns{
    \begin{aligned}
        m_{t} &= \beta m_{t-1} + d_t,    \\
        x_{t+1} &= x_t - \alpha m_{t},
    \end{aligned}
    &&
    \text{with } 
    d_t &= \left\{\begin{aligned}
\tilde g_t                              & \quad \text{ for gradient descent,} \\
\tilde g_t / \norm{\tilde g_t}_2        & \quad \text{ for normalized GD,} \\
\sign(\tilde g_t)                       & \quad \text{ for sign descent.} \\
    \end{aligned}\right.   
}
For Adam, we use the standard implementation in PyTorch \citep{paszke2019pytorch}.

\subsection{Class distribution for common datasets and tokenizers}
\label{apx:tokenizer}

\cref{fig:token-histogram}
provides additional examples 
of the heavy-tailed distribution of tokens
using the basic english tokenizer in \texttt{torchtext} \citep{paszke2019pytorch},
Byte-Pair Encoding \citep[BPE,][]{senrich2016bpe,gage1994new}
and Unigram \citep{kudo2018unigram}
on the PTB and WikiText-2 datasets.
The relationship between the relative frequency rank $k$ and  and the relative frequency 
$\pi_k$ is roughly $\pi_k \propto 1/k$.%

\begin{figure}[!ht]
\centering
\includegraphics[width=.9\textwidth]{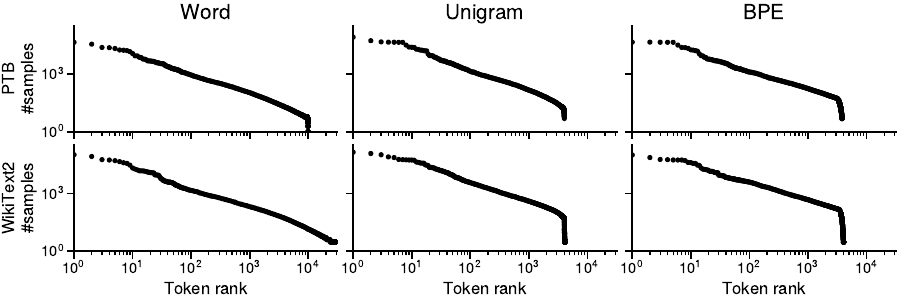}
\caption{
\textbf{Different tokenizers and datasets lead to heavy-tailed token distributions.}
Comparison of word and subword tokenization (BPE, Unigram)
on the PTB and WikiText2 datasets.
}
\label{fig:token-histogram}
\end{figure}

\clearpage
\subsection{Validation}
\label{apx:fig1-validation}

In \cref{fig:1-validation}, we show the validation error on the same problem 
as \cref{fig:1}, training GPT2-Small on WikiText-103.
The validation loss exhibits the same separation across class frequencies, 
and the faster progress of Adam on low-frequency classes is also visible.
While this trend does not hold for all the settings we investigate, 
as some settings use smaller datasets and deterministic training 
to isolate the source of the training difficulties, 
the benefit of Adam on low-frequency classes does not immediately lead to overfitting.

\begin{figure}[!ht]
\centering
\includegraphics{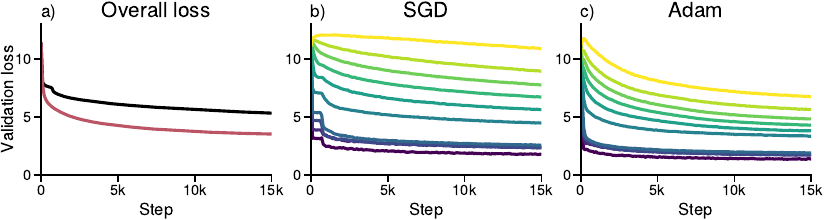}
\caption{
\textbf{The class-separation behavior of \cref{fig:1} holds on the validation loss.}
Same experiment as \cref{fig:1}, training GPT2-Small on WikiText-103, but showing the validation loss.
(a) Distribution of the classes sorted by class frequency, split into groups corresponding to ${\approx}$10\% of the data. 
(b) Overall validation loss. (c, d) Validation loss for each group using SGD and Adam. 
SGD makes little to no progress on low-frequency classes while Adam makes progress on all groups. 
(b) is the average of (c, d) for the respective optimizer.
}
\label{fig:1-validation}
\end{figure}

\clearpage
\section{Heavy-tailed imbalance on vision datasets}
\label{apx:vision-resnet}
\label{apx:vision}

This section gives additional results on vision tasks to complement \cref{sec:mnist}.
\begin{itemize}[leftmargin=1.0em,label={-},parsep=2pt,itemsep=0pt,topsep=0pt]
    \item \cref{fig:lnresnet} shows a similar behavior on a ResNet18 with LayerNorm instead of BatchNorm.
    \item \cref{fig:balancedvit} shows a similar behavior with a vision transformer.
    \item \cref{fig:barcoded-mnist-only} confirms that GD can solve the barcoded MNIST variant without imbalance.
\end{itemize}

In \cref{fig:lnresnet}, we use the same settings \cref{fig:resnet}.
training a ResNet18 on a uniform and unbalanced subset of ImageNet,
but replace the normalization layers with LayerNorm \citep{ba2016layernorm}
instead of BatchNorm \citep{ioffe2015batchnorm}.
We observe a similar pattern as in \cref{fig:resnet}.
Although Adam slightly outperforms SGD on the uniform dataset, 
the performance gap grows on the imbalanced one.

\begin{figure}[!ht]
\centering
\prefigspace{}
\includegraphics{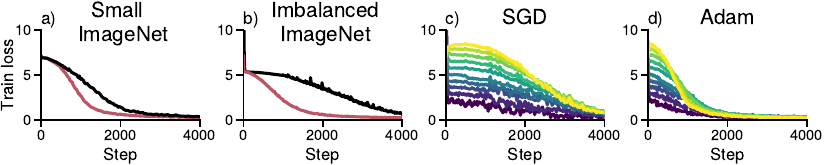}
\includegraphics{figs/plots/paper/legend/plot_legends_0_stoch.pdf}
\postfigspace{}%
\vspace{-1.5em}
\caption{%
\textbf{Adam outperforms SGD on ResNet with LayerNorm under heavy-tailed imbalance.}
(a)~Performance on a uniform subset of ImageNet
(b) and on an imbalanced subset with class frequencies $\pi_k \propto 1/k$.
(c, d)~Performance of GD and Adam across frequencies.
}
\label{fig:lnresnet}
\end{figure}

In \cref{fig:balancedvit}, 
we train a vision transformer on the ImageNet dataset, without subsampling,  
to confirm that the training behavior is similar.
While vision transformers might require more data or regularization than their ResNet counterparts
to achieve comparable generalization performance, 
the optimization problem does not appear to be more difficult for SGD than for Adam.

\begin{figure}[!ht]
\centering
\begin{minipage}[t]{.4\textwidth}
~\\[-.5em]
\centering
\includegraphics[width=.8\textwidth]{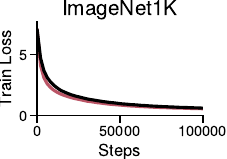}
\end{minipage}%
\begin{minipage}[t]{.6\textwidth}
\caption{%
\textbf{Adam and SGD perform similarly training a Vision Transformer with balanced Classes.}
Training loss on the full ImageNet dataset (without subsampling).
There is little performance in training performance.}
\label{fig:balancedvit}
\end{minipage}
\end{figure}

In \cref{fig:imbalancedvit}, 
we train the same vision transformer on the uniform and imbalanced subsets of ImageNet. 
As in prior experiments with vision data, 
the performance of Adam appears unaffected by the change in class frequencies
while the performance of SGD degrades.

\begin{figure}[!ht]
\centering
\prefigspace{}
\includegraphics{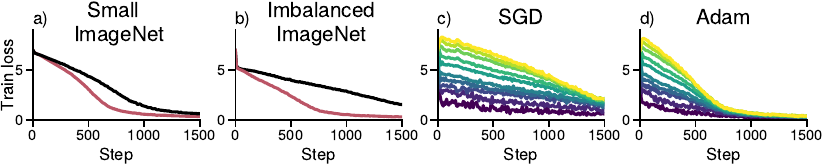}
\includegraphics{figs/plots/paper/legend/plot_legends_0_stoch.pdf}
\postfigspace{}%
\vspace{-1.5em}
\caption{%
\textbf{Adam outperforms SGD on vision transformer under heavy-tailed imbalance.}
(a)~Performance on a uniform subset of ImageNet
(b) and on an imbalanced subset with class frequencies $\pi_k \propto 1/k$.
(c, d)~Performance of GD and Adam across frequencies.
}
\label{fig:imbalancedvit}
\end{figure}

\clearpage
\subsection{Sanity checks on Barcoded MNIST}

\cref{fig:image} in \cref{sec:mnist} showed that the performance gap between GD and Adam
on the imbalanced variant of MNIST with barcoded images is larger than on plain MNIST.
In this section, we verify that the training difficulties 
encountered on the CNN on the imbalanced MNIST dataset of \cref{fig:image} are indeed due to class imbalance. 
As we create new images and new classes by adding a barcode in the corner of existing images, 
it could be that the dataset becomes harder to fit. 

In \cref{fig:barcoded-mnist-only}, 
we run Adam and GD to train the same network on the MNIST dataset only, the barcoded-only subset of the imbalanced MNIST 
and the combination of the two, leading to an imbalanced dataset.
While Adam is faster GD on the barcoded-only dataset, 
both algorithms reach negligible error within 200 steps. 
In contrast, on the combined imbalanced dataset MNIST+Barcoded,
GD fails to make progress on the low-frequency classes and stalls.

\begin{figure}[H]
    \centering
    \includegraphics[width=.32\textwidth]{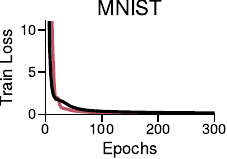}
    \includegraphics[width=.32\textwidth]{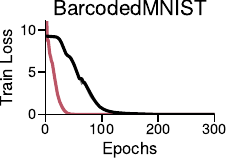}
    \includegraphics[width=.32\textwidth]{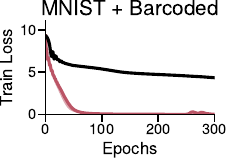}
    \caption{\textbf{GD optimizes on balanced barcoded data.} 
    Training a CNN on only the barcoded portion of the data, which has balanced classes.
    While Adam is slightly faster, both optimizers reach negligible error within 200 steps. As the level of imbalance is increased, GD performs increasingly worse than Adam.}
    \label{fig:barcoded-mnist-only}
\end{figure}

\clearpage

\section{Additional examples on linear models with class imbalance}
\label{apx:more-info}

In \cref{sec:linear}, we showed that GD already becomes slow 
in the presence of class imbalance. 
In this section, we give additional details. 
We discuss the impact of the input distribution, 
as class imbalance alone is technically not sufficient to make GD slow, 
and show that GD eventually does converge.

\subsection{Impact of input distribution}\label{apx:data-need-correlation}

Imbalance alone is not sufficient to induce slow performance of GD on low-frequency classes.
It is possible to generate a dataset with heavy-tailed class imbalance where GD fits all classes fast,
by making the inputs $\vx_i$ (close to) orthogonal, $\lin{\vx_i, \vx_j} \approx 0$ for $i \neq j$. 

In the proof of \cref{thm:gflow} in \cref{apx:continuous-time}, 
we use the independence across classes to show that the classes are learned at a rate $\propto 1/\pi_k$.
If \emph{all} the samples are orthogonal, $\lin{\vx_i, \vx_j} = 0$ for every $i,j$, 
a similar decomposition can show that each sample will be learned independently of all the other, 
at a speed that does not depend on a class frequency. 

Note that it is also not necessary for all samples to be independent of each other.
In the simple setting used in \cref{thm:gflow}, 
samples from the same class are collinear
while samples from separate class are independent
A mixture model where samples from the same class are aligned ($\abs{\lin{\vx_i, \vx_j}} > \delta$ if $y_i=y_j$)
but independent otherwise ($\abs{\lin{\vx_i, \vx_j}} \leq \epsilon$ if $y_i \neq y_j$),
as the setting of \citet{feldman2020does} would also exhibit class separation.

To avoid this issue in \cref{fig:linear}, 
we draw the inputs from a high-dimensional uniform distribution on $[0,1]^d$,
ensuring that for any two samples $\vx_i, \vx_j$, $\lin{\vx_i, \vx_j} > 0$.
If we were to sample data from $\Normal{0, 1}^d$ in sufficiently high dimension, 
the samples can be independent enough to avoid the slowdown due to class imbalance. 
We illustrate this in \cref{fig:input-data}, 
where we use a smaller synthetic data
with inputs drawn from $\Normal{1, 1}$ (left) and $\Normal{0,1}$ (right).
The zero-mean data, which is be approximately orthogonal as $d > n$, 
does not exhibit a slow progress on low-frequency classes.

Note that if the linear model uses a bias term, 
this effectively makes the samples more aligned, 
as it is equivalent to adding a dimension to the input data where each sample has the same value. 
The behavior of GD on aligned data appears to be a better representation 
of the behavior of GD on language transformers,
as we observe a performance separation per class frequency on GD,
even when tuning only the last layer of a language transformer in \cref{fig:opts}.
Although the embedding are initialized to be zero-mean Gaussian noise, 
the embedding representation of the tokens in transformer are aligned, 
and this alignment increases with depth \citep[e.g.]{noci2022signal}. 

\begin{figure}[H]
\centering
\begin{subfigure}[t]{\textwidth}
\centering
\caption{\textbf{Aligned data -- sampled from $\Normal{1,1}^d$}}
\includegraphics[width=.8\textwidth]{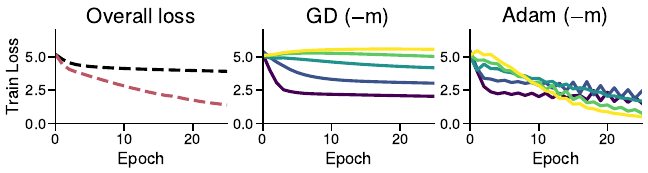}
\end{subfigure}
\begin{subfigure}[t]{\textwidth}
\centering
\caption{\textbf{Independent data -- sampled from $\Normal{0,1}^d$}}
\includegraphics[width=.8\textwidth]{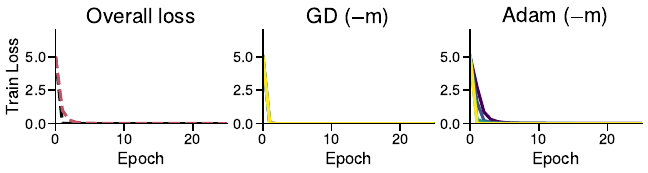}
\end{subfigure}
\caption{%
\textbf{The distribution of the inputs can have a large impact on the performance.}
Linear softmax regression on the \RHTL{} dataset, 
but with inputs sampled from $\Normal{1,1}$ (a) and $\Normal{0, 1}$ (b).
}
\label{fig:input-data}
\end{figure}

\subsection{An early iteration problem}

The observed behavior that GD is slower than Adam at fitting the low-frequency classes, 
might make it seem that GD does not fit the low-frequency classes at all.
Of course, when run for longer, GD converges and fits all classes, as shown in \cref{fig:training-longer}.
This highlight that the difference between the algorithms 
is primarily a difference at the start of training. 
However, this ``start'' can be quite long on large problems, 
as in the transformer of \cref{fig:1},
the average loss on $10\%$ of the data corresponding to the least frequent classes 
is still higher than at initialization after $15$k steps.

\begin{figure}[H]
\includegraphics[width=.33\textwidth]{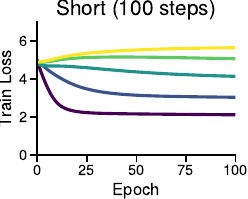}\hfill%
\includegraphics[width=.33\textwidth]{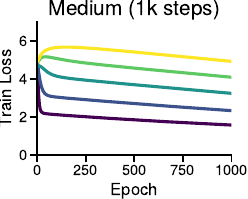}\hfill%
\includegraphics[width=.33\textwidth]{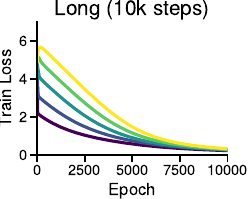}\
\caption{%
\textbf{Training with GD eventually drives the loss down for all classes.}
Training loss over time, with the same step-size, 
for different time horizons (100, 1k, 10k full gradient steps).
GD eventually drives the loss down for all classes, 
but the loss for the least-frequent classes 
only goes below the loss at initialization after 1k steps.
}
\label{fig:training-longer}
\end{figure}

\clearpage
\section{Stochasticity is not necessary to reproduce the gap}
\label{apx:deterministic-transformer}

In \cref{sec:optims}, we argued that the qualitatively different behavior on low-frequency classes
between SGD and Adam in \cref{fig:1} is not due to stochasticity.
In this section, we provide additional results showing that 
this behavior appears across multiple batch sizes on language transformers of different sizes 
and that it can be reproduced in the deterministic setting. 

In \cref{fig:smaller-transformer-stochastic}, we show that a similar qualitative behavior appears 
when training a smaller model (2-layer transformer) on a smaller dataset (PTB).
In \cref{fig:smaller-transformer-deterministic}, 
we repeat the experiment with a 1-layer transformer, trained in full batch on TinyPTB (the validation set of PTB).
The separation between GD and Adam on low-frequency classes in the deterministic settings 
is also visible in \cref{fig:image,fig:linear,fig:opts,fig:grad-hess} in the main paper.
These results indicate that stochasticity it is not necessary to reproduce the behavior observed in \cref{fig:1}. 

\begin{figure}[!ht]
    \centering
    \includegraphics[scale=0.9125]{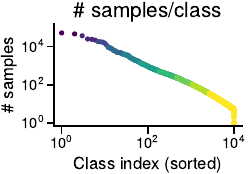}
    \includegraphics[scale=0.9125]{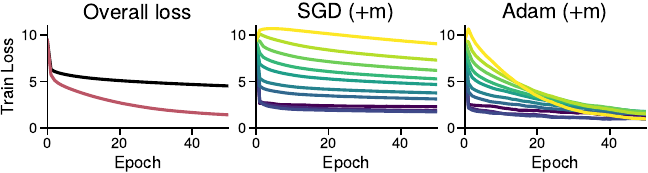}\\
    \includegraphics{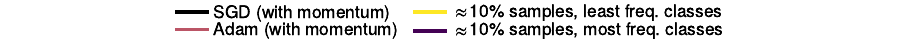}
    \vspace{-1em}
    \caption{\textbf{Similar behavior as \cref{fig:1} on a smaller problem.} 
    Training a 2-layer transformer on PTB with Adam and SGD using larger batch-sizes.
    As in \cref{fig:1}, 
    SGD makes little to no progress on low-frequency classes while Adam makes progress on all subsets.
    \textbf{(a)} Distribution of the classes and subsets of the data sorted by class frequency,
    each corresponding to ${\approx}10\%$ of the samples.
    \textbf{(b)} Overall training loss. 
    \textbf{(c, d)} Training loss for each subset for SGD and Adam.
    \textbf{(b)} is the average of \textbf{(c, d)}.} 
    \label{fig:smaller-transformer-stochastic}
\end{figure}

\begin{figure}[!ht]
    \centering
    \includegraphics[scale=0.9125]{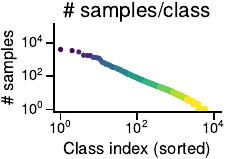}
    \includegraphics[scale=0.9125]{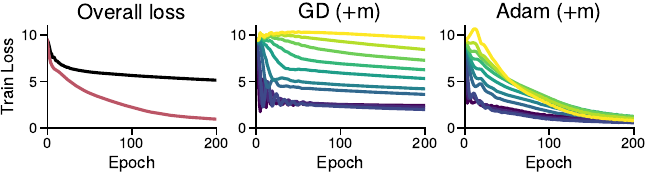}\\
    \includegraphics{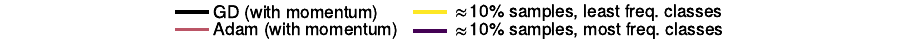}
    \vspace{-1em}
    \caption{\textbf{Similar behavior as \cref{fig:1} on a one-layer transformer with deterministic updates.} 
    Trained on TinyPTB.
    As in \cref{fig:1}, 
    GD makes little to no progress on low-frequency classes while Adam makes progress on all subsets.
    \textbf{(a)} Distribution of the classes and subsets of the data sorted by class frequency,
    each corresponding to ${\approx}10\%$ of the samples.
    \textbf{(b)} Overall training loss. 
    \textbf{(c, d)} Training loss for each subset for SGD and Adam.
    \textbf{(b)} is the average of \textbf{(c, d)}.} 
    \label{fig:smaller-transformer-deterministic}
\end{figure}

\clearpage
\section{Comparing normalized GD and sign descent on additional problems}
\label{apx:more-opt}
\label{apx:additional-opts}

\cref{fig:opts} in \cref{sec:optims} we compared 
GD and Adam to normalized GD and sign descent on 
the last layer of a one-module transformer on TinyPTB, 
showing that Adam and sign descent perform similarly.
We repeat this experiment on other settings here 
to confirm that sign descent leads to similar benefits as Adam on low-frequency classes,
and that changing the direction, as in sign descent, 
has more impact than just changing the magnitude, as in normalized GD.

We also observe this behavior on the following problems:
\begin{itemize}[leftmargin=1.0em,label={-},parsep=2pt,itemsep=0pt,topsep=0pt]
    \item \cref{fig:linear-more-opts}: A linear model on \textbf{\RHTL{}}, as in \cref{fig:linear}.
    \item \cref{fig:tinyptb-more-opts}: A one-module transformer on \textbf{TinyPTB}, as in \cref{fig:smaller-transformer-deterministic}, training all layers.
    \item \cref{fig:cnn-barcode-more-opts}: A CNN on \textbf{MNIST+Barcoded}, as in \cref{fig:image}.
\end{itemize}

\begin{figure}[H]
    \centering
    \includegraphics[width=\textwidth]{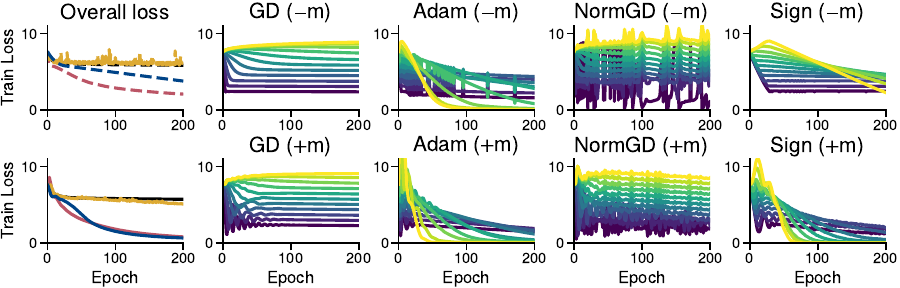}
    \vspace{-1.5em}
    \caption{\textbf{All optimizers on the linear model of \cref{fig:linear}.}}
    \vspace{0.5em}
    \label{fig:linear-more-opts}
    ~\\
    \includegraphics[width=\textwidth]{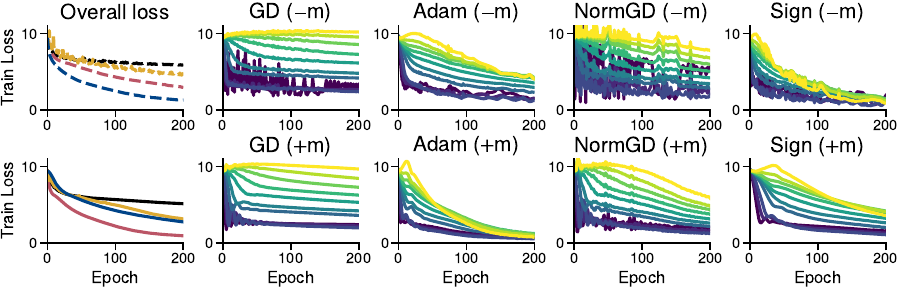}
    \vspace{-1.5em}
    \caption{\textbf{All optimizers on the transformer of \cref{fig:smaller-transformer-deterministic}.}}
    \vspace{0.5em}
    \label{fig:tinyptb-more-opts}
    ~\\
    \includegraphics[width=\textwidth]{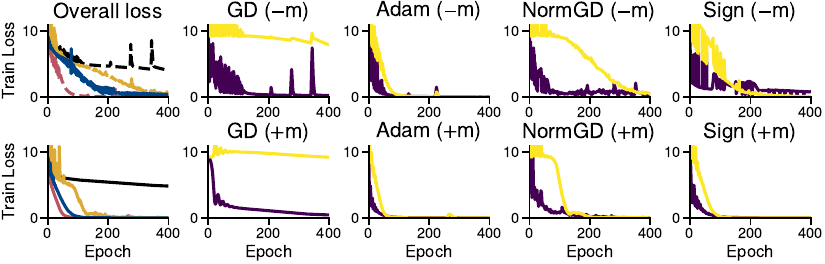}
    \includegraphics{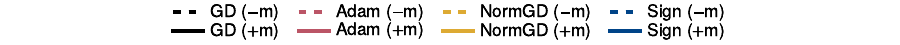}
    \vspace{-1.5em}
    \caption{%
        \textbf{All optimizers on the CNN of \cref{fig:image}.}  
        First column: Overall training loss.
        Remaining: Loss by frequency groups for each optimizer, with and without momentum ($+$m, bottom/$-$m, top).
    }
    \label{fig:cnn-barcode-more-opts}
\end{figure}

\clearpage
\section{Up-weighting low-frequency classes can improve the performance of SGD}
\label{apx:upweight}

To support \cref{sec:optims}, 
we show show that upweighting low-frequency classes helps reduce the performance gap between SGD and Adam 
on problems with heavy-tailed class imbalance, 
providing evidence that the optimization difficulties are associated with class imbalance.

While reweighting the loss of samples from class $k$ by $\nicefrac{1}{\pi_k}$ to address the class imbalance seems intuitive,
optimizing the reweighted loss is no longer guaranteed to lead to progress on the original loss, 
especially if the weights are large. 
Indeed, we find that on some problems this reweighting does not improve performance
(although SGD and Adam perform similarly on the reweighted loss, not shown).
However, the less extreme reweighting of $\nicefrac{1}{\!\sqrt{\pi_k}}$
appears to consistently outperform SGD.

In \cref{fig:reweighting}, 
we run SGD on the reweighted loss 
with the two weighting schemes, $1/\pi_k$ and $1/\sqrt{\pi_k}$
and plot its performance on the original, unweighted loss.
We compare the performance of the two reweighting schemes with SGD and Adam,
all with momentum, on the following 4 problems.
\begin{itemize}[leftmargin=1.0em,label={-},parsep=2pt,itemsep=0pt,topsep=0pt]
    \item The small transformer on PTB in \cref{fig:smaller-transformer-stochastic} (stochastic training)
    \item The Linear model on synthetic data in \cref{fig:linear} (deterministic training)
    \item The CNN on the imbalanced MNIST dataset in \cref{fig:image} (deterministic training)
    \item The ResNet18 on the imbalanced ImageNet dataset in \cref{fig:resnet} (stochastic training)
\end{itemize}

We found that the combination of both Adam and reweighting
did not improve over  running Adam on the original loss
and do no include it in \cref{fig:reweighting}.

\begin{figure}[!h]
\centering
\includegraphics[width=\textwidth]{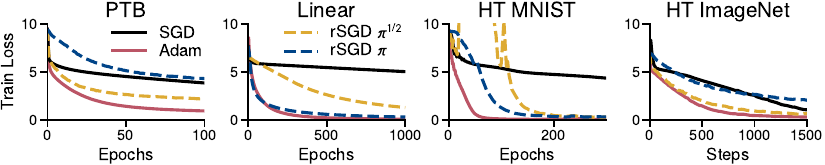}%
\caption{\textbf{Reweighting the loss improves the performance of SGD on low-frequency classes.}
The plots show the unweighted loss, while SGD and Adam optimize a reweighted loss.
Reweighted SGD (rSGD) with weights $\nicefrac{1}{\sqrt{\pi_k}}$ consistently outperforms plain SGD, 
although it can lead to spikes, as on the CNN on the MNIST dataset.
Reweighting with weights $\nicefrac{1}{\pi_k}$
is sometimes better (Linear, MNIST)
but can be worse (PTB, ImageNet) as it optimizes a different objective.
}
\label{fig:reweighting}
\end{figure}

\clearpage
\section{Dynamics of the gradient and Hessian throughout training}
\label{apx:detailled-dynamics}

This section provides additional details on the dynamics of (S)GD and Adam 
discussed in \cref{sub:toy_grad_hess}.

\begin{itemize}[leftmargin=1.0em,label={-},parsep=2pt,itemsep=0pt,topsep=0pt]
    \item \cref{fig:grad-hess-apx}
    shows the dynamics of GD and Adam 
    on the linear model on synthetic data in \cref{fig:linear} (deterministic training), 
    and additionally shows the average predicted probabilities $p$ for each frequency group, 
    showing that the deviation from the linear relationship for 
    rare classes coincides with the predicted probabilities $p$ for those classes going to $1$.
    \item The following figures show the correlation on additional problems, 
    on 
    \begin{itemize}[leftmargin=1.0em,label={-},parsep=2pt,itemsep=0pt,topsep=0pt]
        \item \cref{fig:grad-hess-apx-ptb}
        The small transformer on PTB in \cref{fig:smaller-transformer-stochastic} (stochastic training)
        \item \cref{fig:grad-hess-apx-mnist}
        The CNN on the imbalanced MNIST dataset in \cref{fig:image} (deterministic training)
        \item \cref{fig:grad-hess-apx-resnet}
        The ResNet18 on the imbalanced ImageNet dataset in \cref{fig:resnet} (stochastic training)
    \end{itemize}
    \item \cref{fig:apx-inverse-correlation} illustrates that this correlation 
    does not hold globally and only emerges throughout training by showing 
    that a \emph{negative} correlation can instead be found 
    by looking at the inverse of the weights, $-\mW_t$, 
    over the path taken by Adam.
\end{itemize}

\begin{figure}[ht]
\centering
\begin{subfigure}[t]{\textwidth}
\caption{\textbf{Dynamics over the path of GD}}
\includegraphics[width=\textwidth]{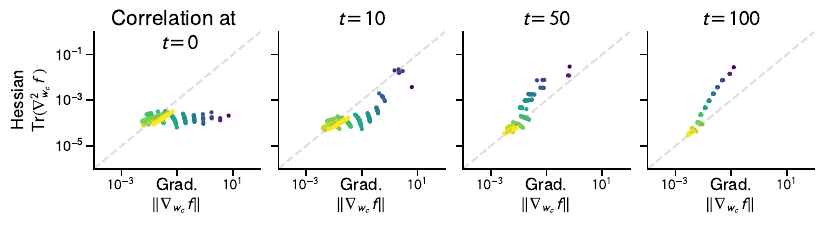}
\end{subfigure}
\vspace{-1.5em}
\begin{subfigure}[t]{\textwidth}
\caption{\textbf{Dynamics over the path of Adam}}
\includegraphics[width=\textwidth]{figs/plots/paper/grad_hessian/LR_BXIY_PC_grad_hess_corr_adam.pdf}
\label{subfig:adam}
\end{subfigure}
\vspace{-1.5em}
\begin{subfigure}[t]{\textwidth}
\caption{\textbf{Predicted probabilities over the course of optimization}}
\includegraphics[width=\textwidth,trim={0 0 0 0.6cm},clip]{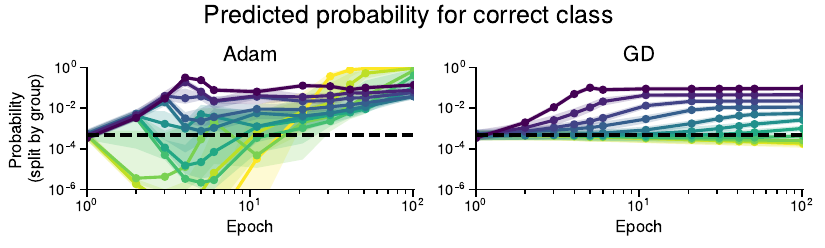}
\label{subfig:probabilities}
\end{subfigure}
\postfigspace{}
\vspace{-1em}
\caption{\textbf{Evolution of the gradient norm and Hessian trace through optimization.}
Taken over the path of GD (a) and Adam (b) on the linear problem of \cref{fig:linear}.
The blocks correspond to the rows $\vw_1, ..., \vw_c$ of the parameter matrix $\mW$.
The color indicates the class frequency, showing that  
lower (higher) frequency classes have smaller (larger) gradient norm and Hessian trace.
\cref{subfig:adam} is a replication of \cref{fig:grad-hess}, given here for convenience.
The deviation from the perfect correlation is explainable by the fact that difference classes are learned at difference speed, 
leading to a different value of $p$ in \cref{prop:assignment}, shown in (c).
For SGD, frequent classes are learned faster than infrequent ones, 
while for Adam, $p$ is similar among the most frequent groups of classes while $p\to 1$ for the least frequent classes.
}
\label{fig:grad-hess-apx}
\end{figure}

\begin{figure}[ht]
\centering
\begin{subfigure}[t]{1.0\textwidth}
\caption{\textbf{Dynamics over the path of GD}}
\includegraphics[width=\textwidth]{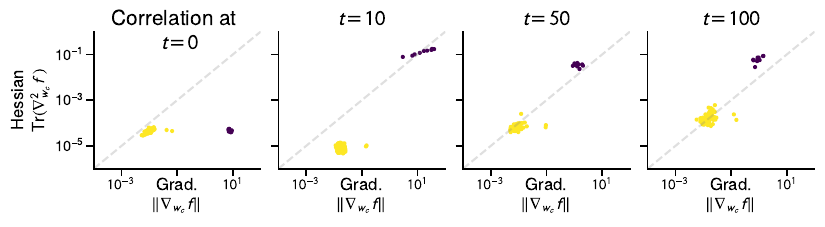}
\end{subfigure}
\begin{subfigure}[t]{\textwidth}
\caption{\textbf{Dynamics over the path of Adam}}
\includegraphics[width=\textwidth]{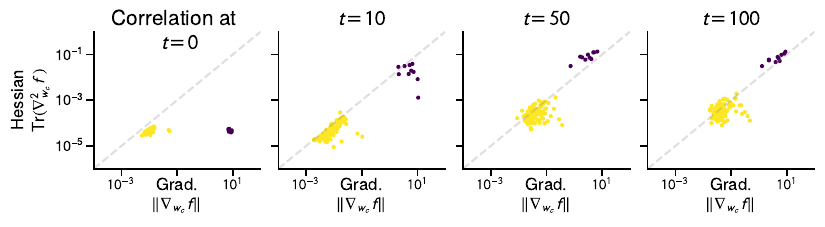}
\end{subfigure}
\vspace{-1em}
\postfigspace{}
\caption{\textbf{Evolution of the gradient norm and Hessian trace through optimization.}
Taken over the path of GD and Adam on the CNN on imbalanced MNIST in \cref{fig:image}.
Note that this problem only has two groups of classes with different frequencies;
$10$ classes have ${\approx}5$k samples while $10$k classes have $5$ samples.
}
\label{fig:grad-hess-apx-mnist}
\end{figure}

\begin{figure}[ht]
\centering
\begin{subfigure}[t]{1.0\textwidth}
\caption{\textbf{Dynamics over the path of SGD}}
\includegraphics[width=\textwidth]{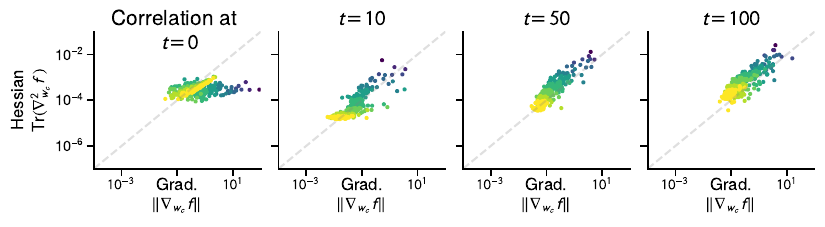}
\end{subfigure}
\begin{subfigure}[t]{\textwidth}
\caption{\textbf{Dynamics over the path of Adam}}
\includegraphics[width=\textwidth]{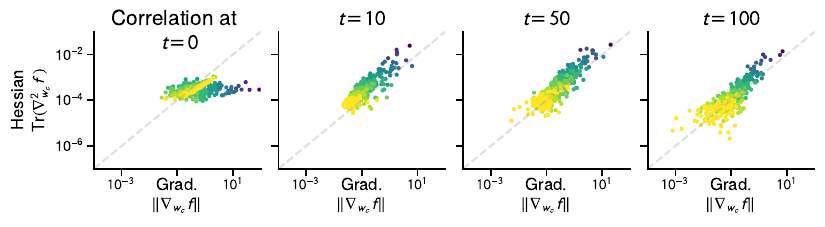}
\end{subfigure}
\postfigspace{}
\caption{\textbf{Evolution of the gradient norm and Hessian trace through optimization.}
Taken over the path of SGD and Adam on the ResNet18 on imbalanced ImageNet in \cref{fig:resnet}.
}
\label{fig:grad-hess-apx-resnet}
\end{figure}

\begin{figure}[ht]
\centering
\begin{subfigure}[t]{1.0\textwidth}
\caption{\textbf{Dynamics over the path of SGD}}
\includegraphics[width=\textwidth]{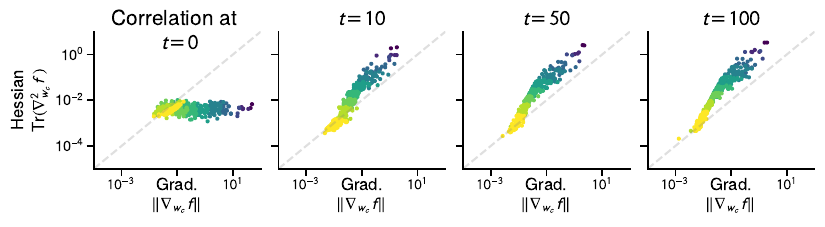}
\end{subfigure}
\begin{subfigure}[t]{\textwidth}
\caption{\textbf{Dynamics over the path of Adam}}
\includegraphics[width=\textwidth]{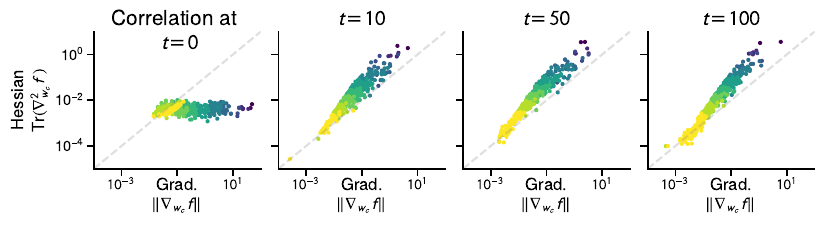}
\end{subfigure}
\postfigspace{}
\caption{\textbf{Evolution of the gradient norm and Hessian trace through optimization.}
Taken over the path of SGD and Adam on the small Transformer on PTB in \cref{fig:smaller-transformer-stochastic}.
}
\label{fig:grad-hess-apx-ptb}
\end{figure}

\begin{figure}[!ht]
\includegraphics{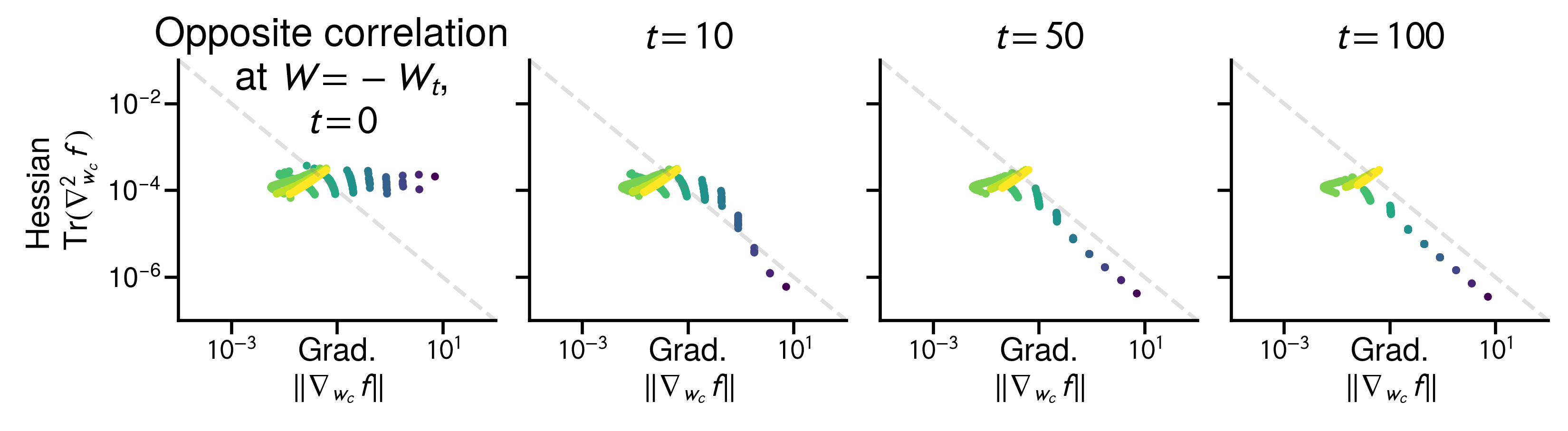}
\caption{\textbf{The correlation only holds while training.} 
Correlation between the gradient and Hessian blocks 
through the path $\{-\mW_t\}$, where $\mW_t$ are the iterates of Adam on the linear model of \cref{fig:linear}.
This illustrates that the correlation described in \cref{prop:assignment}
is not a global property of the problem and requires that the optimizer make progress and assign samples to their correct classes.
}
\label{fig:apx-inverse-correlation}
\end{figure}

\clearpage
\section{Correlation between the gradient and Hessian across blocks}
\label{apx:theory}

This section gives the proof of \cref{prop:assignment} 
in \cref{sub:toy_grad_hess}
\vspace{-1em}
\propassignment*
\vspace{-1em}
The requirement that the class frequencies do not vanish, $\pi_k = \omega(\nicefrac{1}{c})$,
is necessary to make it possible to discuss class frequencies as $c \to \infty$, 
unless the class frequencies do not depend on $c$.
While the frequencies $\pi_k$ and the number of classes $c$ can be independent, 
for example if $\pi_k$ follows an exponential decay, $\pi_k \propto 2^{-k}$,
it does not hold for all distributions.
While it may seem that this result only holds for relatively frequent classes,
as it requires $\pi_k c \to \infty$,
we can see that nearly all the data comes from classes where this correlation holds
when the classes are distributed as $\pi_k \propto 1/k$. 
Denote by $H(c) = \sum_{k=1}^c 1/k = \Theta(\log c)$. 
After normalization, we have $\pi_k = \nicefrac{1}{k H(c)}$.
The correlation result holds as long as $\pi_k c \to \infty$, 
and so it at least holds for the first $k \leq {c}/{\log(c){}^2}$ classes as $\pi_k c \geq \log(c) \to \infty$.
While this only cover a $1/\log(c)^2$ fraction of the classes, 
those classes account for nearly all the data as
\aligns{
    \sum_{k=1}^{\ceil{\frac{c}{\log(c)}}} \pi_k = \frac{H\paren{\ceil{c / \log(c)^2}}}{H(c)}
    = \Theta\paren{\frac{\log(c) - 2\log\log(c)}{\log(c)}} \to 1.
}

\begin{proof}[Proof of \cref{prop:assignment}]
We first recall the gradient and Hessian for each block $\vw_1, ..., \vw_c$;
\aligns{
    \nabla_{\vw_k} \ell(\mW\!, \vx, \vy)
    = \textstyle
    (\ind{y=k} - \vp(\vx)_k)\vx,
    &&
    \nabla_{\vw_k}^2 \ell(\mW\!, \vx, \vy)
    = \textstyle
    \vp(\vx)_k (1 - \vp(\vx)_k)\vx \vx^\top\!,
}
and the definitions 
of the moments of the data, per class and overall.
\aligns{
    \bar \vx^k  = \textstyle \frac{1}{n_k}\sum_{i=1 : y_i = k}^n \vx_i, &&
    \bar \vx    = \textstyle \frac{1}{n}\sum_{i=1}^n \vx_i, &&
    \bar \mH^k  = \textstyle \frac{1}{n_k}\sum_{i=1 : y_i = k}^n \vx_i\vx_i^\top\!, &&
    \bar \mH    = \textstyle \frac{1}{n}\sum_{i=1}^n \vx_i\vx_i^\top\!.
}
Our first step is to rewrite the sums for the gradient and Hessian
to separate the influence of the samples of the correct class $k$ and the other samples.
\aligns{
    \nabla_{\vw_k} \Loss(\mW)
    &= 
    \frac{1}{n}\sum_{i=1}^n (\ind{y_i=k}-\vp(\vx_i)_k)\vx_i,
    \\
    \tag{Split by class}
    &= 
    \frac{1}{n}\sum_{j=1}^c \sum_{i:y_i=j} (\ind{y_i=k}-\vp(\vx_i)_k)\vx_i,
    \\
    \tag{Use class frequencies $\pi_j = n_j/n$}
    &= 
    \sum_{j=1}^c \frac{\pi_j}{n_j}\sum_{i:y_i=j} (\ind{y_i=k}-\vp(\vx_i)_k)\vx_i,
    \\
    &= 
    \pi_{k} \frac{1}{n_k}\sum_{i=1 : y_i = k}^n (1-\vp(\vx_i)_k)\vx_i
    +\!\!\!\!
    \sum_{j=1, j\neq k}^c  \frac{\pi_j}{n_j}\sum_{i:y_i=j}  (-\vp(\vx_i)_k)\vx_i.
    \\
    \nabla_{\vw_k}^2 \Loss(\mW)
    &= \frac{1}{n} \sum_{i=1}^n \vp(\vx_i)_k (1 - \vp(\vx_i)_k)\vx_i \vx_i^\top,
    \\
    &= 
    \frac{\pi_k}{n_k} \sum_{i:y_i=k}  \vp(\vx_i)_k (1 - \vp(\vx_i)_k)\vx_i \vx_i^\top
    + \!\!\!\!
    \sum_{j=1, j \neq k}^c  \frac{\pi_j}{n_j}  \sum_{i:y_i=j}  \vp(\vx_i)_k (1 - \vp(\vx_i)_k)\vx_i \vx_i^\top\!.
}
We can simplify the first terms
using the assumption that $p(\vx_i)_k = p$ for samples of the correct class,
\aligns{
    \frac{\pi_{k}}{n_k}\sum_{i=1 : y_i = k}^n (1-\vp(\vx_i)_k)\vx_i
    = (1-p)\pi_{k} \bar\vx^k,
    &&
    \frac{\pi_k}{n_k} \sum_{i:y_i=k} \vp(\vx_i)_k (1 - \vp(\vx_i)_k)\vx_i \vx_i^\top
    = p(1-p) \pi_k \bar\mH^k.
}
We introduce the following shorthands for the second terms,
\aligns{
    \vd_k = 
    c \sum_{j=1, j\neq k}^c \frac{\pi_j}{n_j}\sum_{i:y_i=j} (-\vp(\vx_i)_k)\vx_i,
    &&
    \mD_k = 
    c \sum_{j \neq k} \frac{\pi_j}{n_j} \sum_{i:y_i=j}  \vp(\vx_i)_k (1 - \vp(\vx_i)_k)\vx_i \vx_i^\top.
}
Using those simplifications, we obtain that 
\aligns{
    \nabla_{\vw_k} \Loss(\mW)
    = 
    (1-p)\pi_{k} \bar\vx^k
    +
    \frac{1}{c} \vd_k,
    &&
    \nabla_{\vw_k}^2 \Loss(\mW)
    = 
    p(1-p) \pi_k \bar\mH^k
    + 
    \frac{1}{c} \mD_k.
}
The terms $\vd_k$, $\mD_k$ are averages of terms weighted 
by $c \vp(\vx_i)_k$, which by assumption is $O(1)$, 
and as such both $\norm{\vd_k}$ and $\Tr(\mD_k)$ are $O(1)$.
The ratio between the two will be dominated 
by the contribution of their first term 
as long as $\pi_k$ dominates $1/c$, 
in the sense that $\lim_{c\to\infty} \frac{1}{\pi_k c} \to 0$, as
\aligns{
    \lim_{c\to\infty}
    \frac{\norm{\nabla_{\vw_k}\Loss}}{\Tr(\nabla^2_{\vw_k}\Loss)}
    &=
    \lim_{c\to\infty}
    \frac{\norm{
        (1-p)\pi_{k} \bar\vx^k
        +
        \frac{1}{c}
        \vd_k
    }}{\Tr(
        p(1-p) \pi_k \bar\mH^k
        +
        \frac{1}{c} \mD_k
    )}
    \\
    &= 
    \lim_{c\to\infty}
    \frac{\norm{
        (1-p) \bar\vx^k
        +
        \frac{1}{c \pi_k}
        \vd_k
    }}{\Tr(
        p(1-p) \pi_k \bar\mH^k
        +
        \frac{1}{c \pi_k} \mD_k
    )}
    = 
    \frac{1}{p}
    \frac{\norm{
        \bar\vx^k
    }}{\Tr(
        \bar\mH^k
    )}.
    \tag*{\qedhere}
}
\end{proof}

\subsection{Off-diagonal blocks are orders of magnitude smaller than diagonal blocks}
\label{apx:off-diag}
Our discussion \cref{sub:toy_grad_hess} ignored the impact of off-diagonal blocks.
In this section, we show that they are small. 
The diagonal and off-diagonal blocks of the matrix for~$k \ne k'$.
\aligns{
    \begin{array}{llllr}
    & \mH_{kk} 
    &\coloneqq \nabla_{\vw_k}^2 \ell(\mW\!, \vx, y)
    &= \vp(\vx)_k (1 - \vp(\vx)_k)\vx \vx^\top\!,
    \\
    \text{ and for } j \neq k, 
    & \mH_{kj} 
    &\coloneqq \nabla_{\vw_k}\!\!\nabla_{\vw_{k'}} \ell(\mW\!, \vx, \vy)
    &= \vp(\vx)_k (\phantom{1} - \vp(\vx)_{k'})\vx \vx^\top\!.
    \end{array}
}
From this, we can see that, on average, the magnitude of the off-diagonal blocks 
will be smaller than that of the diagonal blocks, as \\[-2em]
\aligns{
    \mH_{kk} = - \sum_{j=1, j \neq k}^c \mH_{kj},
}
because $\sum_{k'=1, k' \neq k}^c \vp(\vx)_k \vp(\vx)_{k'} = \vp(\vx)_k(1-\vp(\vx)_k)$,
This means that the matrix $\mT : [c \times c]$ 
formed by taking the trace of the blocks, $\mT_{jk} = \Tr(\mH_{jk})$, is diagonally dominant.

\cref{fig:off-diag-blocks,fig:off-diag-lin} show that the magnitude of the entries of the Hessian 
in off-diagonal blocks is orders of magnitude smaller than those of the diagonal blocks.
Instead of plotting the $[cd \times cd]$ Hessian, 
we subsample $40$ classes and $40$ input dimensions 
and plot the resulting $[160 \times 160]$ entries
at different points throughout the trajectory of Adam 
on the problem of \cref{fig:linear}.
\cref{fig:off-diag-blocks} shows the matrices with classes sampled uniformly 
and \cref{fig:off-diag-lin} with classes sampled log-uniformly

\begin{figure}[ht]
\includegraphics{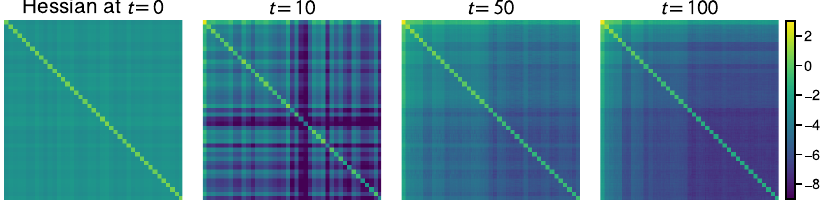}
\vspace{-1.5em}
\caption{\textbf{The off-diagonal blocks are much smaller than the diagonal blocks.}
Showing the magnitude $\log_{10}(\abs{(\nabla^2\Loss)_{ij}})$ for a $[160 \times 160]$ subset of the Hessian,
sampling 40 classes and 40 input dimensions uniformly.
}
\label{fig:off-diag-lin}
\end{figure}

\clearpage
\section{Continuous time GD and sign descent on a simple imbalanced problem}
\label{apx:continuous-time}

We give the proof of \cref{thm:gflow} on the simple imbalanced setting, restated here for convenience.

\simplestImbalancedSetting{}

\thmgflow*

We separate the proof for gradient flow into 3 parts.
\cref{lem:dynamics} simplifies the dynamics into smaller, independent differential equations, 
\cref{lem:solving} solves the differential equation 
and \cref{lem:loss} bounds the loss.
We treat continuous time sign descent separately in \cref{lem:sign}.

\textbf{Notation.}
If $\mW$ is a $[a \times b]$ matrix, 
then $\vw_1, ..., \vw_a$ are the rows
and $\vw^1, ..., \vw^b$ are the vectors, 
and $w_{ij}$ is the entry at the $i$th column, $j$th row.
For brevity, we use $z = c-1$ 
as the term appears often. 

\begin{lemma}[Separation of the dynamics]
\label{lem:dynamics}
The dynamics of the parameter matrix $\mW$ 
separate into $c$ $2$-dimensional differential equations,
$w_{kk}(t) = a_k(t)$ and $w_{jk}(t) = b_k(t)$ for $j \neq k$, 
where
\aligns{
    a_k(0) &= 0,
    &
    \ddt a_k &= \pi_k \paren{1 - \frac{\exp(a_k)}{\exp(a_k) + (c-1)\exp(b_k)}},
    \\
    b_k(0) &= 0,
    & 
    \ddt b_k &= \pi_k\paren{\hphantom{1} - \frac{\exp(b_k)}{\exp(a_k) + (c-1)\exp(b_k)}}.
}
\end{lemma}
\begin{proof}
Our goal is to simplify the dynamics 
starting at $\mW(0) = 0$ and following the gradient flow 
$\ddt \mW = - \nabla \Loss(\mW)$, where $\mW : [c \times d]$.
For the simplified setting, we have that $d = c$ are the inputs 
are the standard basis vectors in $\R^c$.
The derivative of $\Loss$ w.r.t. a single element $w_{kj}$ is
\aligns{
    \partial_{w_{kj}} \Loss(\mW)
    = - \pi_k \ind{k=j}
    + \pi_j\sigma(\vw^j)_k.
}
As $\partial_{w_{kj}}$ only depends on $\vw^j$ for all $k$, 
The dynamics are independent across the columns of $\mW$, 
giving $c$ independent equations in $\R^c$,
\aligns{
    \vw^j(0) = 0,
    &&
    \ddt \vw^j
    &=
    \pi_j (\ve_j - \sigma(\vw^j)).
}
To further simplify the dynamics, 
we use the fact that the weights that are not associated with the correct class 
have the same dynamics.
For any indices $i,j$ different from $k$, $w_{ik}(t) = w_{jk}(t)$.
They have the same derivatives if they have the same value, as 
\aligns{
    - \ddt w_{ik} &= \pi_k \sigma(\vw^k)_i
    = \pi_k \frac{\exp(w_{ik})}{\sum_{k'} \exp(w_{k'k})}
    = \pi_k \frac{\exp(w_{jk})}{\sum_{k'} \exp(w_{k'k})}
    = \pi_k \sigma(\vw^k)_j 
    = - \ddt w_{jk},
}
so they will have the same dynamics and 
the equation can be reduced to a system of 2 variables, 
$w_{kk} = a_k$ and $w_{jk} = b_k$ for any $j\neq k$,
with 
\aligns{
    a_k(0) &= 0,
    &
    \ddt a_k &= \pi_k \paren{1 - \frac{\exp(a_k)}{\exp(a_k) + (c-1)\exp(b_k)}},
    \\
    b_k(0) &= 0,
    & 
    \ddt b_k &= \pi_k\paren{\hphantom{1} - \frac{\exp(b_k)}{\exp(a_k) + (c-1)\exp(b_k)}}.
    \tag*{\qedhere}
}
\end{proof}

\begin{lemma}[Solution of the dynamics]\label{lem:solving}
For a given class with frequency $\pi$, 
the dynamics of the parameters $a$ and $b$ in \cref{lem:dynamics} evolve as follows, 
using the shortcuts $f(t) = 1 + c\pi t$ and  $z = c-1$,
\aligns{
    a(t) 
    =
    \frac{1}{c}
    &\paren{f(t)
    - z W\paren{
        \frac{1}{z}
        \exp\paren{\frac{1}{z}f(t)}
    }
    }
    &&
    b(t) = -\frac{1}{z}a(t),
}
\end{lemma}
\begin{proof}
We want the solution to the differential equation
\aligns{
    a(0) = 0 &&
    \ddt a &= \pi \paren{1 - \frac{\exp(a)}{\exp(a) + (c-1)\exp(b)}},
    \\
    b(0) = 0 &&
    \ddt b &= \pi \paren{\hphantom{1} - \frac{\exp(b)}{\exp(a) + (c-1)\exp(b)}}.
}
The general solution, ignoring the initial conditions,
uses the Lambert $W$ function and constants $K_1, K_2$.\footnote{WolframAlpha solution for $\pi = 1$: 
\url{https://www.wolframalpha.com/input?i=d/dt+x(t)+=+1-exp(x(t))/(exp(x(t))+c*exp(y(t))),+d/dt+y(t)+=+-exp(y(t))/(exp(x(t))+c*exp(y(t)))}
}
For brevity, we introduce the shortcut $z = c-1$.
\aligns{
    a(t) 
    =&&
    \hphantom{K_1 - }
    \frac{1
    }{zc} 
        &\paren{
            c e^{-K_1}K_2 + cz \pi t 
            - z^2 W\paren{
                \frac{1}{z}
                \exp\paren{
                    \frac{
                        c
                    }{z^2}
                    \paren{z \pi t + e^{-K_1} K_2}
                    - K_1
                }
            }
        },
    \\
    b(t) 
    =&& K_1
    - \frac{1
    }{z^2c} 
    &\paren{
        c e^{-K_1}K_2 + cz \pi t 
        - z^2 W\paren{
            \frac{1}{z}
            \exp\paren{
                \frac{c}{z^2}
                \paren{z \pi t + e^{-K_1} K_2}
                - K_1
            }
        }
    }.
}
We need to set $K_1, K_2$ to satisfy the initial conditions $a(0)=b(0)=0$.
As $b(t) = K_1 - a(t)/z$, we must have that $K_1$ = 0, giving the simplification 
\aligns{
    a(t) 
    =
    \frac{1
    }{zc} 
        &\paren{
            c K_2 + cz \pi t 
            - z^2 W\paren{
                \frac{1}{z}
                \exp\paren{
                    \frac{
                        c
                    }{z^2}
                    \paren{z \pi t + K_2}
                    - K_1
                }
            }
        },
    &&
    b(t) 
    = - \frac{1}{z} a(t).
}
To set $K_2$, we need to have
\aligns{
    0 = zc a(0) = 
    c K_2 
    - z^2 W\paren{
        \frac{1}{z}
        \exp\paren{
            K_2 \frac{c}{z^2}
        }
    }
    \implies
    W\paren{\frac{1}{z}\exp\paren{K_2\frac{c}{z^2})}}
    = \frac{c}{z^2} K_2
}
Since $W(xe^x) = x$ for $x>0$, 
the equation is satisfied for $K_2 = \frac{z}{c}$, 
as we get $W\paren{\frac{1}{z}e^\frac{1}{z}}=\frac{1}{z}$,
giving
\aligns{
    a(t) 
    =
    \frac{1}{c}
    &\paren{1 + c\pi t 
    - z W\paren{
        \frac{1}{z}
        \exp\paren{\frac{1}{z}\paren{1+c\pi t}}
    }
    }
    &&
    b(t) = -\frac{1}{z}a(t).
    \tag*{\qedhere}
}
\end{proof}

\begin{lemma}[Bound for the loss]\label{lem:loss}
For $t$ sufficiently large such that $1 + c \pi_k t \geq z \log z + 1$,
\aligns{
    \ell_k(t) 
    = 
    \Theta\paren{
        \frac{1}{\pi_k t}
    }.
}
\end{lemma}

Using the simplification derived in \cref{lem:dynamics}
and the solution of the differential equation in \cref{lem:solving}, 
we can rewrite the loss for a specific class as a function of time as 
\aligns{
    L_k(\mW)
    &\coloneqq
    -\log(\sigma(\mW\ve_k)_k) = -\log\paren{
         \frac{
            \exp\paren{w_{kk}}
        }{
            \sum_{j=1}^c \exp\paren{w_{jk}}
        }
    },
    \\
    \ell_k(t) \coloneqq L_k(\mW(t))
    &= -\log\paren{
         \frac{
            \exp\paren{a_k(t)}
        }{
            \exp(a_k(t)) + (c-1)\exp(b_k(t))
        }
    }
    = \log\paren{1+(c-1)\exp(cb_k(t))},
}
where the equality uses that $a_k(t) = (c-1) b_k(t)$.
For brevity, we will drop the index $k$ in $a_k$, $b_k$, $\ell_k$ and $\pi_k$
and use the shortcut $z = c-1$, bounding the quantity
\aligns{
    \ell(t) = \log\paren{1 + z\exp(cb(t))}.
}    
Expanding the definition of $b(t)$ using \cref{lem:solving}, we have 
\aligns{
    z\exp(cb(t))
    =
    z\exp\paren{
        -\frac{1}{z}\paren{f(t) - z W\paren{\frac{1}{z}\exp\paren{\frac{1}{z}f(t)}}}
    },
    && \text{ where } && f(t) = 1 + c\pi t.
}
To simplify the $W$ function, we use the fact that for $x>e$ \citep[Theorem 2.7]{hoorfar2008lambert}
\aligns{
    W(x) = \log(x) -\log(\log(x)) + \delta(x) 
    && \text{ where } &&
    \frac{1}{2}\leq\delta(x) \frac{\log(x)}{\log(\log(x))} \leq\frac{e}{e-1}.
}
To use this bound on $W\paren{\frac{1}{z}\exp\paren{\frac{1}{z}f(t)}}$, 
we need $\frac{1}{z}\exp\paren{\frac{1}{z}f(t)} \geq e$, 
which is satisfied for $t$ sufficiently large, once $f(t) \geq z (\log z + 1)$.

Using that $\log\paren{\frac{1}{z}\exp\paren{\frac{1}{z}f(t)}} = \frac{1}{z}f(t) -\log(z)$, 
and writing $h(t) =\delta\paren{\frac{1}{z}\exp\paren{\frac{1}{z}f(t)}}$, we have
\aligns{
    f(t) - z W\paren{\frac{1}{z}\exp\paren{\frac{1}{z}f(t)}}
    &=
    f(t) - z \paren{
        \frac{1}{z}f(t)
        - \log(z) 
        - \log\paren{\frac{1}{z}f(t)-\log(z)}
        + h(t)
    },
    \\
    &=
    z \paren{
        \log\paren{f(t) - z\log(z)} - h(t)
    },
}
giving the simplification  
\aligns{
    z \exp(c b(t)) &= 
    z\exp\paren{
        -\frac{1}{z}\paren{f(t) - z W\paren{\frac{1}{z}\exp\paren{\frac{1}{z}f(t)}}}
    },
    \\
    &=
    z\exp\paren{-\log\paren{f(t) - z\log(z)} + h(t)}
    =
    \frac{z \exp(h(t))}{f(t) - z \log z},
}
This gives the average loss
\aligns{
    \ell(t) = \log\paren{1 + z\exp(cb(t))}
    = \log\paren{1+\frac{z \exp(h(t))}{f(t) - z \log z}}
}

To bound this expression, we can use that 
$\frac{z \exp(h(t))}{f(t) - z \log z} \geq 0$ after 
$f(t) \geq z\log z$, which we have already assumed to apply 
the bound on the $W$ function, 
and use the bounds $\frac{x}{1+x}\leq \log(1+x)\leq x$ to get 
\aligns{
    \frac{z \exp(h(t))}{f(t) - z \log z + z \exp(h(t))}
    \leq
    \ell(t)
    \leq
    \frac{z \exp(h(t))}{f(t) - z \log z}.
}
As $h(t)$ is upper bounded by a constant and $\lim_{t\to\infty}h(t) = 0$, $\lim_{t\to\infty} \exp(h(t)) = 1$, we have 
\aligns{
    \ell(t) =
    \Theta\paren{
        \frac{z}{f(t) - z \log z}
    }
    =
    \Theta\paren{
        \frac{1}{\pi t}
    }.
    \tag*{\qedhere}
}

\begin{lemma}
\label{lem:sign}
The loss at time $t$ for continuous time sign descent 
is $\ell_k(t) = \log\paren{1+ (c-1) \exp(-ct)}$
\end{lemma}
\begin{proof}
The same decomposition as in \cref{lem:dynamics} hold, with the dynamics
\aligns{
    a_k(0) = 0,
    &&
    \ddt a_k = 1,
    &&
    a_k(t) = t,
    &&&&
    b_k(0) = 0,
    && 
    \ddt b_k = -1,
    &&
    b_k(t) = -t,
}
leading to the following loss 
\aligns{
    \ell_k(t) = \log\paren{1+ (c-1) \exp(-ct)}
    = \Theta(z\exp(-ct)).
    \tag*{\qedhere}
}
\end{proof}

\end{document}